\documentclass{article}
\usepackage{amssymb}
\usepackage{enumitem}
\usepackage{amsthm}
\usepackage{amsmath}

\usepackage{bm}
\usepackage{xcolor}

\usepackage{tikz}
\usetikzlibrary{shapes,arrows.meta,positioning,calc,decorations.pathreplacing,
                patterns,backgrounds,fit,arrows}

\newcommand{\Uepi}{U_{\mathrm{epi}}}
\newcommand{\Uale}{U_{\mathrm{ale}}}

\usepackage{microtype}
\usepackage{graphicx}
\usepackage{subcaption}
\usepackage{multirow}
\usepackage[table]{xcolor}
\usepackage{adjustbox}
\usepackage{booktabs} 
\usepackage{hyperref}
\usepackage{algorithm}
\usepackage{algorithmic}

 \usepackage[main, final]{neurips_2026}

\usepackage{makecell}
\usepackage{amsmath}
\usepackage{amssymb}
\usepackage{mathtools}
\usepackage{amsthm}
\usepackage{microtype}  
\usepackage{csquotes}
\usepackage{natbib}

\usepackage[capitalize,noabbrev]{cleveref}

\usepackage{pgfplots}
\pgfplotsset{compat=1.17}
\usepackage{amsmath,amssymb}
\usepackage{xcolor}

\definecolor{epistemic}{RGB}{102, 166, 30}     
\definecolor{aleatoric}{RGB}{228, 134, 34}     
\definecolor{baseline}{RGB}{204, 102, 102}     
\definecolor{ours}{RGB}{148, 103, 189}         
\definecolor{target}{RGB}{200, 230, 200}       
\definecolor{encoder}{RGB}{100, 149, 237}      

\definecolor{credalrow}{rgb}{0.996, 0.957, 0.8}  
\definecolor{baselinerow}{RGB}{245, 245, 245}  
\definecolor{oursrow}{RGB}{227, 242, 253}      
\definecolor{bestcell}{RGB}{200, 230, 201}     

\newcommand{\rPdz}[1]{#1}
\newcommand{\vcGN}[1]{#1}
\newcommand{\rAxY}[1]{#1}

\newtheorem{theorem}{Theorem}[section]

\newtheorem{lemma}[theorem]{Lemma}
\newtheorem{corollary}[theorem]{Corollary}
\newtheorem{definition}[theorem]{Definition}

\usepackage[textsize=tiny]{todonotes}
\usepackage{tikz}
\usetikzlibrary{positioning, arrows.meta, shapes.geometric}
\title{Structural Separation of Epistemic and Aleatoric Uncertainty across Supervised Latent Variable Models}

\author{
Tanmoy Mukherjee\textsuperscript{1}, 
Marius Kloft,\textsuperscript{2}, 
Pierre Marquis\textsuperscript{1,3}, 
Zied Bouraoui\textsuperscript{1} 
\\
\textsuperscript{1} Univ. Artois, CNRS, CRIL, France, \textsuperscript{2} RPTU Kaiserslautern-Landau\\
\textsuperscript{3} Institut Universitaire de France
\\
\texttt{mukherjee@cril.fr}\\
}

\begin{document}

\maketitle

\begin{abstract}
Predictive uncertainty is commonly decomposed into epistemic and aleatoric components that are useful if they support different decisions:epistemic uncertainty should indicate reducible model ignorance, whereas aleatoric uncertainty should indicate persistent ambiguity, suggesting review, abstention, or ambiguity-aware prediction.
Standard decompositions often produce strongly correlated estimates because both quantities are derived from the same predictive distribution. We study a design principle, \emph{structural separation}, which assigns epistemic and aleatoric uncertainty to disjoint parameter paths trained with distinct supervision targets: reducible prediction error for epistemic uncertainty and persistent label ambiguity for aleatoric uncertainty. We instantiate this principle across Credal Concept Bottleneck Mode and show a gradient-isolation result indicating that the two uncertainty heads are not coupled through shared training gradients under the proposed parameterization. Across five ambiguity-aware benchmarks, structural separation substantially reduces epistemic-aleatoric correlation while preserving predictive performance. Further analysis shows that aleatoric estimates track annotator- or corpus-derived ambiguity, while epistemic estimates are more sensitive to prediction error and data availability. These results suggest that supervised latent-variable architectures provide a practical route toward uncertainty estimates that are not merely decorrelated but operationally distinguishable.
\end{abstract}

\section{Introduction}
\label{sec:introduction}
Uncertainty estimates are useful only insofar as they support different decisions. Epistemic uncertainty (EU) is usually interpreted as reducible model ignorance: it should be high where the model lacks evidence and decrease as relevant data are observed. Aleatoric uncertainty (AU), by contrast, reflects persistent ambiguity in the data or annotation: it can remain high even for a well-trained model when the input admits multiple plausible labels. Operationally, high EU calls for more evidence or model improvement, whereas high AU calls for abstention, ambiguity-aware prediction, or human review.

Despite this conceptual distinction, standard predictive uncertainty methods often fail to separate the two quantities empirically. Monte Carlo dropout~\citep{gal2016dropout}, deep ensembles~\citep{lakshminarayanan2017simple}, evidential networks~\citep{sensoy2018evidential}, and related methods derive both estimates from the same predictive object. Recent benchmarks report that the resulting components are often strongly correlated across datasets~\citep{mucsanyi2024benchmarking}. This is not merely an implementation artifact: when both quantities are functions of the same predictive distribution, they inherit a shared algebraic dependency. Recent formal results show that different underlying epistemic and aleatoric situations can induce the same optimal predictive distribution, making them indistinguishable to any decomposition that only observes that distribution~\citep{tomov2025illusion}. Thus, even accurate predictors may fail to provide uncertainty estimates that support different decisions (see Appendix~\ref{sec:related} for a fuller discussion of related work).

We study an alternative design principle: \emph{structural separation} (Fig.~\ref{fig:impossibility}). Rather than deriving EU and AU post hoc from the same predictive distribution, we assign them to distinct supervised latent channels. EU is trained to track reducible prediction error, while AU is trained to track persistent label ambiguity, such as annotator disagreement or corpus-derived answer distributions. The goal is not to claim that architectural separation alone guaranties semantic correctness, but to remove a source of algebraic coupling and then test whether the resulting channels behave differently under epistemic and aleatoric diagnostics. Fig.~\ref{fig:impossibility} illustrates the contrast: predictive decompositions constrain both signals through $p(y\mid x)$, whereas structural separation allows them to vary independently by assigning them to distinct supervised paths.

We instantiate this principle in \emph{supervised latent-variable models} (SLVMs): architectures with an intermediate layer of individually supervised units, such as concept bottleneck models (CBMs)~\citep{koh2020concept}, self-explaining neural networks (SENNs)~\citep{alvarez2018towards}, and prototype-based models~\citep{chen2019looks}. These architectures are natural candidates for structural separation because their intermediate variables can be supervised, inspected, and routed through separate parameter paths. Our primary instantiation is a Credal CBM in which EU and AU estimates are produced by disjoint heads trained with distinct losses.
Because neither uncertainty estimate is constrained to be a function of $p(y\mid x)$ alone, the construction avoids the specific setting targeted by the impossibility result of \citet{tomov2025illusion}. Avoiding this post-hoc setting is not sufficient on its own: gradient-isolated heads can still produce statistically correlated outputs or decorrelated outputs that are semantically uninformative. We therefore test three claims separately: (i) \emph{architectural}, through gradient isolation; (ii) \emph{statistical}, through low EU/AU correlation; and (iii) \emph{semantic}, through the alignment of each head with its intended validation target.

Our contributions are fourfold. First, we study structural separation for SLVMs, separating EU and AU into distinct supervised parameter paths rather than deriving both from the same predictive distribution. Second, we instantiate this principle within a Credal CBM and show a gradient-isolation result: with disjoint parameterization and separate loss terms, the EU and AU heads remain uncoupled, as they do not share training gradients. Third, we introduce diagnostics for evaluating whether separation is meaningful, including gradient-isolation probes, EU/AU correlation, data-scaling reducibility, and ambiguity-tracking analyzes. Finally, we verify a cross-architecture and cross-benchmark evaluation showing that structural separation reduces EU/AU coupling while preserving predictive performance and producing uncertainty signals aligned with distinct validation targets.

\section{Background and Problem Setup}
\label{sec:background}

We introduce the uncertainty decomposition, explain why predictive-distribution-based decompositions are structurally coupled, and define the class of SLVMs used for structural separation. 

\noindent\textbf{Notation} 
We denote inputs as $x \in \mathcal{X}$, task labels as $y \in \{1,\ldots,J\}$,
and concept labels as $c \in \{1,\ldots,K\}^C$ for $C$ concepts with $K$
classes each (e.g., $K=3$ for negative/unknown/positive).  We write $\Delta^{K-1}$ for the $(K-1)$-simplex
of probability distributions. A key distinction: $p(c|x)$ denotes the model's
predicted distribution over concepts, while $p^*(c|x)$ denotes the true
(unknown) distribution reflecting genuine label ambiguity---when multiple
annotators disagree, $p^*$ has high entropy. For uncertainty, $U_{\mathrm{epi}}$
denotes epistemic uncertainty (model ignorance about $p^*$) and
$U_{\mathrm{ale}}$ denotes aleatoric uncertainty (entropy of $p^*$ itself).

\begin{figure*}[t]
\centering
\includegraphics[width=0.9\textwidth]{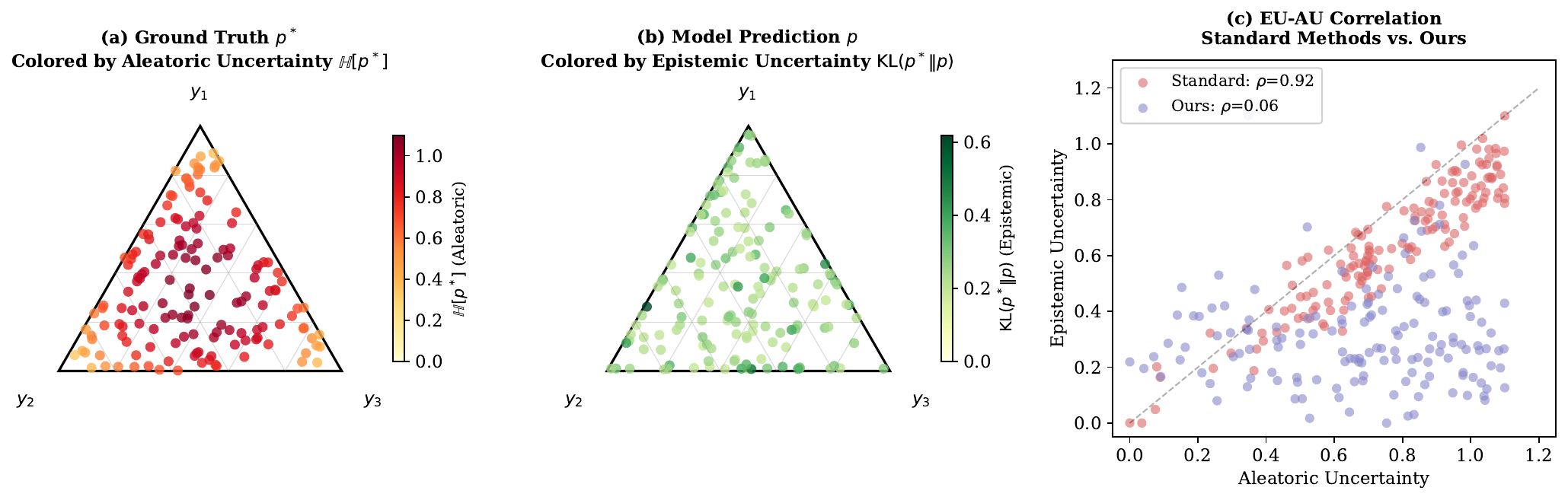}
\caption{\textbf{Why standard decomposition fails.} 
Aleatoric uncertainty reflects where $p^*$ lies on the simplex: ambiguous
cases cluster near the center \textbf{(left)}. Epistemic uncertainty reflects
how far $p$ deviates from $p^*$ \textbf{(middle)}. These are geometrically
independent properties, yet methods deriving both from $p$ produce estimates
that fall along a diagonal \textbf{(right, red)}---the ``algebraic trap.''
A structurally separated SLVM (our framework) recovers
the independence \textbf{(right, blue)}.}
\label{fig:impossibility}
\end{figure*}

\subsection{Uncertainty Decomposition}
\label{sec:uncertainty_definitions}
Predictive uncertainty arises from distinct sources, each necessitating a different approach. For a given input $x$, we identify: \textbf{epistemic uncertainty} $U_{\mathrm{epi}}(x)$, which results from \emph{limited data or model capacity} and is \emph{reducible} through the acquisition of additional data or the development of improved models; and \textbf{aleatoric uncertainty} $U_{\mathrm{ale}}(x)$, which stems from inherent \emph{data ambiguity} and remains \emph{irreducible}, even with an optimal model and infinite data. For instance, annotators may disagree on whether the term "interesting food" conveys a positive or neutral sentiment.
Following \citet{tomov2025illusion}, let $p^*(\cdot \mid x)$ denote the ground-truth conditional distribution that would be obtained in the limit of infinitely many competent annotators \footnote{$p^*$ is unobserved in practice; we use the empirical
annotator distribution $\widehat{p^*}$ as supervision wherever a learning
signal on $p^*$ is required (\S\ref{sec:method}). We continue to write
$p^*$ in this section for clarity of exposition.} and let $p(\cdot \mid x)$ denote the model’s predictive distribution. We then define
\begin{align}
\label{eq:eu-au-defs}
    U_{\mathrm{epi}}(x) &= D_{\mathrm{KL}}\bigl(p^*(\cdot \mid x) \,\big\|\, p(\cdot \mid x)\bigr), &
    U_{\mathrm{ale}}(x) &= H\bigl[p^*(\cdot \mid x)\bigr].
\end{align}
\textbf{The algebraic trap.}
Standard estimators derive both $U_{\mathrm{ale}}$ and  $U_{\mathrm{epi}}$ from the same predictive probabilities or sample statistics. Consequently, these estimates remain highly coupled, with a benchmark of nineteen methods reporting Spearman correlations of $\rho \geq 0.78$. This "algebraic trap" arises because marginal predictions $p(\cdot \mid x)$ cannot uniquely distinguish between model confusion and intrinsic ambiguity. While heteroscedastic models and Epistemic Neural Networks attempt to separate these signals, at least one remains linked to $p(y \mid x)$. To resolve this, we introduces supervised latent variable models (SLVMs), which structurally decouple these uncertainties using distinct parameter subsets and supervision.

\subsection{Supervised Latent Variable Models}
\label{sec:slvm}
\begin{definition}[SLVM]
\label{def:slvm}
A \emph{supervised latent variable model} (SLVM) is a triple
$M = (f_{\mathrm{enc}}, g_{\mathrm{lat}}, f_{\mathrm{task}})$ with encoder
$f_{\mathrm{enc}} : \mathcal{X} \to \mathbb{R}^d$, latent map
$g_{\mathrm{lat}} : \mathbb{R}^d \to \mathcal{Z}$, and task head
$f_{\mathrm{task}} : \mathcal{Z} \to \mathcal{Y}$, satisfying:
\textup{(i)} the latent $z = g_{\mathrm{lat}}(h)$ is supervised by an
external label $c$ distinct from $y$; \textup{(ii)} $f_{\mathrm{task}}$
consumes $z$ (not $h$), so the supervised latent layer lies on the forward
path to $y$; and \textup{(iii)} the supervision on $z$ provides a per-unit
signal, decomposable into a point component (which value $z$ should take)
and an ambiguity component (how concentrated the per-unit label
distribution is).
\end{definition}


CBMs~\citep{koh2020concept} are the main example: $z$ is a vector of concept predictions, each supervised by annotations, and the task head predicts $y$ from these concepts. SENNs~\citep{alvarez2018towards} and prototype-based models~\citep{chen2019looks} follow the same pattern when their intermediate units are externally supervised. SLVMs aid uncertainty separation because their intermediate supervision can exceed a single label. CBMs supply this supervision: multi-annotator concept labels induce disagreement distributions estimating $p^*$, directly supervising AU. If three of five annotators label a concept \textsc{as Positive} and two \textsc{as Unknown}, then $\hat{p}^* = (0, 0.6, 0.4)$ with $\mathbb{H}[\hat{p}^*] = 0.67$, which supervises $\sigma_{\mathrm{ale}}$.

\usetikzlibrary{positioning, arrows.meta, shapes.geometric}

\begin{figure*}[!tbp]
\centering
\resizebox{\textwidth}{!}{%
\begin{tikzpicture}[
    node distance=8mm and 10mm,
    every node/.style={font=\normalsize},
    box/.style={
        rectangle, rounded corners=2pt, draw, thick,
        minimum width=22mm, minimum height=10mm, align=center,
        inner sep=3pt
    },
    encbox/.style    ={box, fill=blue!12,    draw=blue!50!black},
    projmu/.style    ={box, fill=red!8,      draw=red!40!black,
                       minimum width=12mm, minimum height=8mm,
                       font=\small},
    projepi/.style   ={box, fill=violet!10,  draw=violet!50!black,
                       minimum width=12mm, minimum height=8mm,
                       font=\small},
    projale/.style   ={box, fill=yellow!18,  draw=orange!60!black,
                       minimum width=12mm, minimum height=8mm,
                       font=\small},
    latbox/.style    ={box, fill=red!12,     draw=red!50!black,
                       minimum width=24mm},
    epibox/.style    ={box, fill=violet!15,  draw=violet!50!black,
                       minimum width=24mm},
    alebox/.style    ={box, fill=yellow!25,  draw=orange!70!black,
                       minimum width=24mm},
    taskbox/.style   ={box, fill=green!12,   draw=green!45!black},
    inputlabel/.style={font=\small, align=center},
    outlabel/.style  ={font=\small\itshape},
    fwd/.style       ={-{Latex[length=2mm]}, thick},
    sup/.style       ={-{Latex[length=2mm]}, thick, dashed,
                       draw=orange!70!black},
    suplab/.style    ={font=\footnotesize\itshape, color=orange!60!black,
                       align=center},
    arrowlab/.style  ={font=\footnotesize\itshape, color=gray!50!black},
    paramtag/.style  ={font=\footnotesize\itshape, color=violet!50!black},
    paramtagale/.style={font=\footnotesize\itshape, color=orange!60!black},
]

\node[inputlabel] (input) {%
    \textit{``Is this movie good?''}\\[1pt]
    {\footnotesize $x$}%
};

\node[encbox, right=10mm of input] (enc) {$f_{\mathrm{enc}}$};
\draw[fwd] (input) -- (enc);
\node[arrowlab, below=0.3mm of enc] {$h \in \mathbb{R}^d$};

\node[projmu,  right=12mm of enc, yshift=14mm] (Wmu)  {$W_\mu$};
\node[projepi, right=12mm of enc]              (Wepi) {$W_{\mathrm{epi}}$};
\node[projale, right=12mm of enc, yshift=-14mm](Wale) {$W_{\mathrm{ale}}$};

\draw[fwd] (enc.east) -- (Wmu.west);
\draw[fwd] (enc.east) -- (Wepi.west);
\draw[fwd] (enc.east) -- (Wale.west);

\node[latbox,  right=10mm of Wmu]  (lat)
    {Concept layer $g_\mu$\\[1pt]
     {\footnotesize $\mu \in \mathbb{R}^K$}};
\node[epibox,  right=10mm of Wepi] (epi)
    {Epistemic head $g_{\mathrm{epi}}$\\[1pt]
     {\footnotesize $\sigma_{\mathrm{epi}}$}};
\node[paramtag, above=0.3mm of epi.north east, anchor=south east]
    {$\phi_{\mathrm{epi}}$};
\node[alebox,  right=10mm of Wale] (ale)
    {Aleatoric head $g_{\mathrm{ale}}$\\[1pt]
     {\footnotesize $\sigma_{\mathrm{ale}}$}};
\node[paramtagale, below=0.3mm of ale.south east, anchor=north east]
    {$\phi_{\mathrm{ale}}$};

\draw[fwd] (Wmu.east)  -- node[arrowlab, above]{$h_\mu$}             (lat.west);
\draw[fwd] (Wepi.east) -- node[arrowlab, above]{$h_{\mathrm{epi}}$}  (epi.west);
\draw[fwd] (Wale.east) -- node[arrowlab, above]{$h_{\mathrm{ale}}$} (ale.west);

\draw[decorate, decoration={brace, amplitude=4pt, mirror},
      thick, draw=gray!60]
    ($(epi.east)+(2mm,0)$) -- ($(ale.east)+(2mm,0)$)
    node[midway, right=3mm, font=\footnotesize\itshape, align=left,
         color=gray!50!black]
    {$\phi_{\mathrm{epi}} \cap \phi_{\mathrm{ale}} = \emptyset$};

\node[taskbox, right=22mm of lat] (task)
    {Task head\\[1pt]
     {\footnotesize $f_{\mathrm{task}}(\bar z)$}};
\draw[fwd] (lat) -- (task);

\node[outlabel, right=4mm of task] (yhat) {$\hat y$};
\draw[fwd] (task) -- (yhat);

\node[suplab, above=8mm of lat] (clab)
    {concept labels $c$\\(point + ambiguity)};
\draw[sup] (clab) -- (lat);

\node[suplab, below=8mm of ale] (Hp)
    {$\mathbb{H}[\widehat{p^*}]$\\(annotator entropy)};
\draw[sup] (Hp) -- (ale);

\node[font=\footnotesize\itshape, color=blue!50!black,
      below=8mm of enc] (frozen) {frozen};
\draw[->, draw=blue!50!black, dashed, thin]
    (frozen.north) -- (enc.south);

\end{tikzpicture}
}
\caption{\textbf{Three disjoint parameter paths produce three uncertainty signals.} An SLVM (Def.~\ref{def:slvm}) instantiated as a CBM. The frozen encoder outputs $h$, split by orthogonal projections $W_\mu, W_{\mathrm{epi}}, W_{\mathrm{ale}}$ (Eq.~\ref{eq:projections}) into three subspaces feeding heads $g_\mu, g_{\mathrm{epi}}, g_{\mathrm{ale}}$. The credal centre $\mu$ flows to the task head; $\sigma_{\mathrm{epi}}$ and $\sigma_{\mathrm{ale}}$ come from disjoint parameter sets ($\phi_{\mathrm{epi}} \cap \phi_{\mathrm{ale}} = \emptyset$). Solid arrows: forward computation; dashed arrows: supervision or frozen-encoder constraints. \emph{Orthogonality, disjointness, and the frozen encoder together yield Theorem~\ref{thm:gradient-separation}.}}
\label{fig:slvm-architecture}
\end{figure*}
\section{Structurally Separated SLVMs}
\label{sec:method}
Predictive decompositions couple uncertainty estimates through $p(y\mid x)$. We instead prove a general SLVM lift lemma for gradient-level separation and specialize it to a Credal CBM with orthogonal projections. Figure~\ref{fig:slvm-architecture} illustrates an SLVM instantiated as a CBM.  The encoder produces representations $h = f_{\mathrm{enc}}(x) \in \mathbb{R}^d$, and an
SLVM additionally maps $h$ to a supervised latent
$z = g_{\mathrm{lat}}(h) \in \mathcal{Z}$ on the forward path to $y$
(Definition~\ref{def:slvm}). The credal set $\mathcal{C} \subset \Delta^{K-1}$ is parameterized by mean
$\mu \in \mathbb{R}^K$, epistemic covariance
$\Sigma_{\mathrm{epi}} \in \mathbb{R}^{K \times K}_{++}$, and aleatoric
variance $\sigma^2_{\mathrm{ale}} \in \mathbb{R}_+$.
We write $\phi_{\mathrm{enc}}, \phi_{\mathrm{lat}}, \phi_{\mathrm{task}}$
for the parameters of the encoder, latent map, and task head, and
$\phi_{\mathrm{epi}}, \phi_{\mathrm{ale}}$ for the disjoint parameter sets of
the epistemic and aleatoric heads ($\phi_{\mathrm{epi}} \cap \phi_{\mathrm{ale}}
= \emptyset$); the corresponding loss terms are $\mathcal{L}_{\mathrm{task}},
\mathcal{L}_{\mathrm{concept}}, \mathcal{L}_{\mathrm{KL}},
\mathcal{L}_{\mathrm{ale}}$. We use $\rho(\cdot, \cdot)$ for Spearman
correlation. Complete notation is in Appendix~\ref{app:notation}. Below, we detail our method.

\subsection{Structural Separation}
\label{sec:structural_separation}

We first formalize the separation principle used throughout the paper.

\begin{definition}[Structural separation]
\label{def:structural-separation}
A model satisfies \emph{structural separation} between epistemic and aleatoric uncertainty if:
    (1) EU and AU are parameterized by disjoint parameter sets
    $\phi_{\mathrm{epi}}$ and $\phi_{\mathrm{ale}}$;
    (2) their learning signals have disjoint gradient sources, so that the EU parameters are updated only by the EU objective and the AU
    parameters only by the AU objective;
    (3) the two heads have distinct loss-target inputs: a reducible-error 
target for EU and an ambiguity target for AU.
\end{definition}

The first two conditions prevent EU and AU from being coupled through shared parameters or shared training gradients. The third condition gives the intended semantics of the two heads. This separation is therefore structural by construction, but its semantic validity must still be tested empirically.

\begin{lemma}[Structural-separation lift]
\label{lem:lift}
Let $M$ be an SLVM per Definition~\ref{def:structural-separation}, and suppose its latent
layer is parameterized by $\phi_{\mathrm{lat}}$ (point parameters),
$\phi_{\mathrm{epi}}$ (controlling
$\Sigma_{\mathrm{epi}}(z\mid h)$), and $\phi_{\mathrm{ale}}$ (controlling
$\sigma_{\mathrm{ale}}(h)$), with
$\phi_{\mathrm{epi}} \cap \phi_{\mathrm{ale}} = \emptyset$. Suppose the
training loss decomposes as
\begin{equation}
\label{eq:slvm_loss}
\mathcal{L} = \mathcal{L}_{\mathrm{task}}(\phi_{\mathrm{enc}}, \phi_{\mathrm{lat}}, \phi_{\mathrm{task}})
\;+\; \mathcal{L}_{\mathrm{concept}}(\phi_{\mathrm{enc}}, \phi_{\mathrm{lat}})
\;+\; \lambda_e \mathcal{L}_{\mathrm{KL}}(\phi_{\mathrm{epi}})
\;+\; \lambda_a \mathcal{L}_{\mathrm{ale}}(\phi_{\mathrm{ale}}),
\end{equation}
with $\mathcal{L}_{\mathrm{KL}}$ depending only on $\phi_{\mathrm{epi}}$ and
$\mathcal{L}_{\mathrm{ale}}$ depending only on $\phi_{\mathrm{ale}}$. Then
\begin{equation}
\nabla_{\phi_{\mathrm{epi}}} \mathcal{L} = \lambda_e \nabla_{\phi_{\mathrm{epi}}} \mathcal{L}_{\mathrm{KL}},
\qquad
\nabla_{\phi_{\mathrm{ale}}} \mathcal{L} = \lambda_a \nabla_{\phi_{\mathrm{ale}}} \mathcal{L}_{\mathrm{ale}},
\end{equation}
and the EU and AU parameters of $M$ satisfy the structural-separation
property of Definition~\ref{def:slvm}.
\end{lemma}

\begin{proof}[Proof sketch]
The four loss terms in (\ref{eq:slvm_loss}) depend on disjoint subsets of the parameter vector, a structural property illustrated by the brace annotation in Fig.~\ref{fig:slvm-architecture}. The linearity of $\nabla$ then collapses $\nabla_{\phi_{\mathrm{epi}}} \mathcal{L}$ and $\nabla_{\phi_{\mathrm{ale}}} \mathcal{L}$ to the single surviving term each.
The complete proof and the conditions under which the assumption holds via stop-gradients are given in Appendix~\ref{app:gradient-isolation}. 
\end{proof}

\noindent\textbf{Remark on the converse.} Lemma~\ref{lem:lift} states a
\emph{sufficient} condition: structural separation is achievable in any
SLVM. We do not claim the converse. Other routes (e.g., auxiliary supervision on a non-latent uncertainty head) may also satisfy Definition~\ref{def:structural-separation}.

\noindent\textbf{Instance: }
The CBM is a type of SLVM (Fig~\ref{fig:slvm-architecture}) with $\mathcal{Z} = \mathbb{R}^{C \times K}$, which are per-concept logits. Here, $c$ are concept labels that have an annotator distribution, and $f_{\mathrm{task}}$ is a linear classifier on the average concept prediction. The annotator distribution is divided into a point label for $\phi_{\mathrm{lat}}$ and an entropy signal for $\phi_{\mathrm{ale}}$. An ensemble or variational parameterization of the concept layer includes $\phi_{\mathrm{epi}}$. Lemma~\ref{lem:lift} is directly applicable, and CBMs are the main focus of our experiments (\S\ref{sec:experiments}). Other examples include SENN~\citep{alvarez2018towards}, which uses $z$ as a learned concept basis; the lift works with the SENN stability regularizer included in $\mathcal{L}_{\mathrm{concept}}$. Prototype networks~\citep{chen2019looks} meet Definition~\ref{def:slvm} but do not naturally support multi-annotator supervision in NLP. Therefore, we compare CBM in the main text and discuss prototypes in App~\ref{app:robustness}


\subsection{Credal Concept Bottleneck Model}
\label{sec:credal_cbm}

We make $z$ a distribution-valued latent rather than a point estimate. The latent map $g_{\mathrm{lat}}$ produces variational parameters
 \begin{equation}
 \label{eq:credal-params}
     g_{\mathrm{lat}}(h) = \bigl(\mu(h),\, \Sigma_{\mathrm{epi}}(h),\, \sigma_{\mathrm{ale}}(h)\bigr),
 \end{equation}
 where $(\mu, \Sigma_{\mathrm{epi}})$ defines a Gaussian variational
 posterior $q_\phi(z \mid h) = \mathcal{N}(\mu, \Sigma_{\mathrm{epi}})$
 which, after softmax-pushforward, induces a \emph{credal set}
 \begin{equation}
     \mathcal{C}_{\tau}(h)
=
\left\{
\mathrm{softmax}(z)
:
(z-\mu(h))^\top \Sigma_{\rm epi}(h)^{-1}(z-\mu(h))
\leq \tau^2
\right\}
\subset \Delta^{K-1}.
 \end{equation}
The concept is represented by a simplex, i.e., a convex set of probability distributions. $\Sigma_{\rm epi}(h)$ characterizes the epistemic geometry of the concept-space ambiguity set $\mathcal{C}(h)$, whereas the separate output $\sigma_{\rm ale}(h)$ lies outside the credal set. The quantities $U_{\mathrm{epi}}$ and $U_{\mathrm{ale}}$ (Eq.~(\ref{eq:eu-au-defs})) are derived from the geometry of $\mathcal{C}(h)$ and from $\sigma_{\mathrm{ale}}(h)$, using the parameters $\phi_{\mathrm{epi}}$ and $\phi_{\mathrm{ale}}$ (Definition~\ref{def:structural-separation}). Because a KL-based term alone cannot ensure that $U_{\mathrm{epi}}$ faithfully captures model error, a hybrid objective function is introduced.

\paragraph{Orthogonal Parameterization}
\label{sec:orthogonal_parameterization}
The construction above is realized through a frozen encoder and
three orthogonally projected heads (Fig.~\ref{fig:slvm-architecture}).
Given input $x$, the encoder produces $h = f_{\mathrm{enc}}(x) \in
\mathbb{R}^d$. An orthogonal projection maps $h$ into subspaces:
\begin{equation}
\label{eq:projections}
    h_\mu = W_\mu h, \quad h_{\mathrm{epi}} = W_{\mathrm{epi}} h, \quad h_{\mathrm{ale}} = W_{\mathrm{ale}} h,
\end{equation}
Only the epistemic and aleatoric subspaces need to be orthogonal,
since these are the two paths whose gradients must be isolated
$\mathcal{L}_{\text{orth}}=\|W_{\rm epi} W_{\rm ale}^{\top}\|_F^2 $
The mean path may share representational content with the task pathway.
Full details are provided in Appendix~\ref{app:method-details}.

\paragraph{Supervision Targets}
\label{sec:supervision_targets}
The two uncertainty heads are trained with different targets. If the credal set is trained only variationally—i.e., if $\Sigma_{\mathrm{epi}}$ is supervised solely via the KL term—then $\sigma_{\mathrm{epi}}$ stays roughly constant across inputs because $\mathcal{L}_{\mathrm{KL}}$ is input-independent. Lemma~\ref{lem:lift} ensures gradient disjointness but not empirical EU-validity. The error-supervision term in Eq.~(\ref{eq:epi-loss}) adds the needed input-dependent signal while preserving the lemma’s disjoint-parameter assumption, since it depends only on $\phi_{\mathrm{epi}}$.

\noindent\emph{Aleatoric loss.}
The aleatoric head is used as a supervision target only (not as input to $g_{\mathrm{ale}}$) and is supervised by the empirical annotator entropy
$\mathbb{H}[\widehat{p^*}^{(c)}]$ for each concept $c$. 
\begin{equation}
\label{eq:ale-loss}
    \mathcal{L}_{\mathrm{ale}} = \frac{1}{C} \sum_{c=1}^{C} \left( \sigma_{\mathrm{ale}}^{(c)} - \mathbb{H}\bigl[\widehat{p^*}^{(c)}\bigr] \right)^2,
\end{equation}
\noindent\emph{Epistemic loss.}
The epistemic head combines error supervision with a Hausdorff KL
regularizer:
\begin{equation}
\label{eq:epi-loss}
\mathcal{L}_{\mathrm{epi}} = \frac{1}{C} \sum_{c=1}^{C} \left( \sigma_{\mathrm{epi}}^{(c)} - \phi\bigl(|\hat p^{(c)} - c^{(c)}|_{\mathrm{sg}}\bigr) \right)^2
\;+\; \beta \cdot D_H^+(\mathcal{C} \,\|\, \mathcal{C}_{\mathrm{prior}}),
\end{equation}
where $\hat p^{(c)} = \sigma(\mu^{(c)})$, $c^{(c)} \in \{0,1\}$ is the
ground-truth concept label, $|\cdot|_{\mathrm{sg}}$ is the stop-gradient,
and $\phi$ scales errors to a fixed range
(Appendix~\ref{app:phi-scaling}).

\noindent\emph{Error supervision} makes $\sigma_{\mathrm{epi}}$
input-dependent: high-error concepts receive large credal sets, and the
stop-gradient prevents the head from ``gaming'' the loss by reducing
predicted error.
\textit{Hausdorff KL} $D_H^+ = \sup_{q \in \mathcal{C}} D_{\mathrm{KL}}(q \| p_0)$
regularizes the credal-set volume toward a prior; for diagonal
$\Sigma_{\mathrm{epi}}$ it admits a closed form \ref{app:hausdroff}.
The training objective is \footnote{The decorrelation penalty is optional and is used only to sharpen an already
separated architecture; App~\ref{app:robustness} reports $\lambda_d=0$ runs showing that
structural separation alone substantially reduces EU/AU correlation}.
\begin{equation}
\label{eq:full-loss}
\mathcal{L} = \mathcal{L}_{\mathrm{task}} + \lambda_c \mathcal{L}_{\mathrm{concept}} + \lambda_e \mathcal{L}_{\mathrm{epi}} + \lambda_a \mathcal{L}_{\mathrm{ale}} + \lambda_o \mathcal{L}_{\mathrm{orth}} + \lambda_d\mathcal{L}_{\text{decorr}}.
\end{equation}

\subsection{Gradient Separation Under Orthogonal Projection}
\label{subsec:theory}

Lemma~\ref{lem:lift} establishes gradient disjointness for any SLVM whose $\phi_{\mathrm{epi}}$ and $\phi_{\mathrm{ale}}$ are disjoint. Orthogonal projection guaranties that the two heads are exposed to non-overlapping features, and freezing the encoder eliminates the last remaining shared parameter.

\begin{theorem}[Gradient separation under orthogonal projection]
\label{thm:gradient-separation}
\vcGN{Let $\mathcal{L}$ be the objective in Eq.~(\ref{eq:full-loss})
with frozen encoder, orthogonal projections (Eq.~\ref{eq:projections}),
and stop-gradient in $\mathcal{L}_{\mathrm{epi}}$
(Eq.~\ref{eq:epi-loss}). Then
\begin{equation*}
\nabla_{\phi_{\mathrm{ale}}} \mathcal{L} = \lambda_a \nabla_{\phi_{\mathrm{ale}}} \mathcal{L}_{\mathrm{ale}}, \qquad
\nabla_{\phi_{\mathrm{epi}}} \mathcal{L} = \lambda_e \nabla_{\phi_{\mathrm{epi}}} \mathcal{L}_{\mathrm{epi}}.
\end{equation*}
That is, the aleatoric parameters receive gradients only from
$\mathcal{L}_{\mathrm{ale}}$, the epistemic parameters only from
$\mathcal{L}_{\mathrm{epi}}$, and no gradient flows between them through
either parameters or features.}
\end{theorem}
The proof can be found in Appendix~\ref{app:gradient-proof}.
The theorem guarantees gradient separation under the stated parameterization but does not show that the learned quantities are semantically valid EU and AU estimates. Semantic validity is evaluated in Section~\ref{sec:experiments} via AU-ambiguity tracking, EU sensitivity to prediction error and data availability, and EU/AU correlation.

\textbf{Connection to the impossibility result.} While \citet{tomov2025illusion} shows that no post-hoc function of $p(y\mid x)$ can separate $U_{\mathrm{epi}}$ and $U_{\mathrm{ale}}$, Theorem~\ref{thm:gradient-separation} shows they can be separated by construction: each uncertainty is read from a distinct parameter set with disjoint feature inputs, neither equal to $p(y\mid x)$. Appendix~\ref{app:decorrelation-corollary} develops an asymptotic decorrelation corollary, the role of an optional explicit decorrelation penalty, and the interaction with output correlation.


\begin{figure*}[t]
\centering
\includegraphics[width=0.8\textwidth]{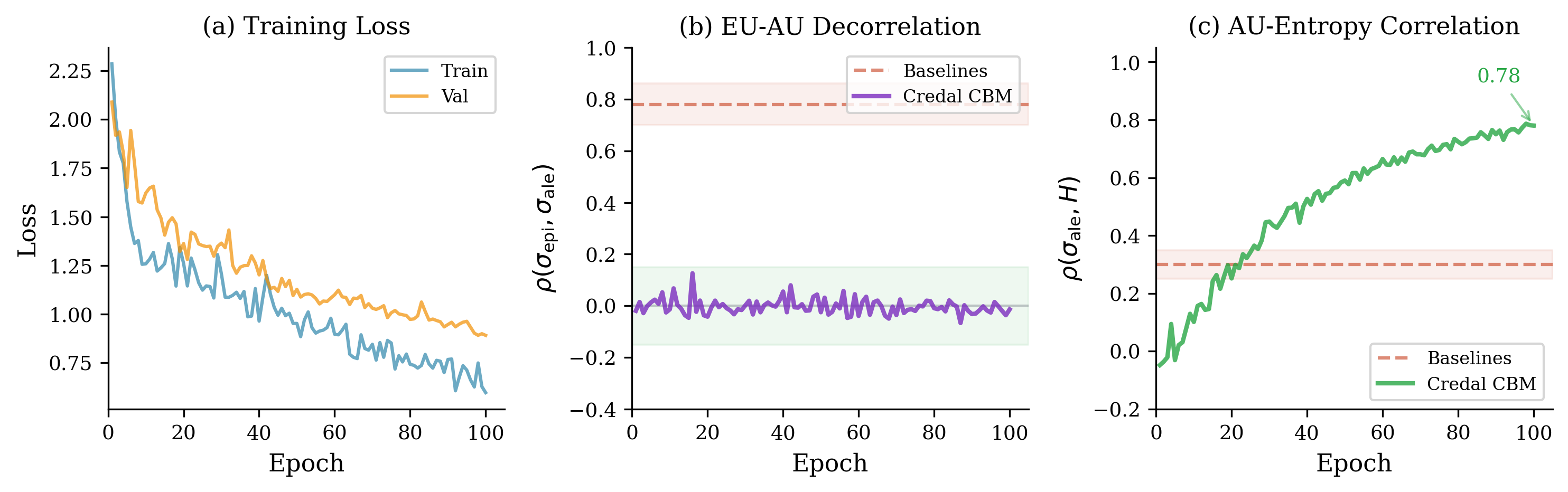}
\caption{\textbf{Training dynamics under 
Structural-Separation} The cross-gradient isolation predicted by Theorem~\ref{thm:gradient-separation} is verified separately
in Appendix~\ref{app:gradient-isolation}; this figure shows the resulting training dynamics and semantic
alignment of the learned uncertainty heads.
(a) Training loss convergence.
(b) EU--AU correlation: baselines maintain coupling
($\rho > 0.7$); our model decorrelates within 10--20 epochs and
stabilises near $\rho \approx 0$.
(c) AU--Entropy on MAQA*: $\sigma_{\mathrm{ale}}$ tracks annotator entropy throughout training, }
\label{fig:theorem-validation}
\end{figure*}

\begin{table*}[t]
\centering
\scriptsize
\setlength{\tabcolsep}{4pt}
\caption{\textbf{Main results across all benchmarks.} Three-tier
comparison: standard post-hoc decomposition, credal-set baselines, and
structurally separated SLVMs (ours). Configuration: $\lambda_d = 5.0$ on the concept-bottleneck datasets and $\lambda_d = 0.1$ on the QA
 datasets, reflecting the difference in AU
 supervision signal: concept-level annotator disagreement vs.\
 corpus-derived answer-distribution entropy
 (Appendix~\ref{app:method-details}). $\rho(U_{\mathrm{epi}},
U_{\mathrm{ale}})$ lower is better; $\rho(U_{\mathrm{epi}},
\mathrm{Err})$, $\rho(U_{\mathrm{ale}}, H)$, and AUROC higher is
better. $H = H[\widehat{p^*}]$ on CEBaB / HateXplain / GoEmotions;
$H = H[p^*]$ on MAQA* / AmbigQA*. 
}
\label{tab:main-results}
\begin{tabular}{@{}llccccc@{}}
\toprule
Dataset & Method & Acc. & $\rho(U_{\mathrm{epi}}, U_{\mathrm{ale}})\downarrow$ & $\rho(U_{\mathrm{epi}}, \mathrm{Err})\uparrow$ & $\rho(U_{\mathrm{ale}}, H)\uparrow$ & AUROC$\uparrow$ \\
\midrule
\multirow{7}{*}{CEBaB}
 & Semantic Entropy   & $81.2$ & $0.81$ & $0.26$ & $0.38$ & $0.67$ \\
 & Deep Ensembles     & $82.1$ & $0.79$ & $0.28$ & $0.41$ & $0.68$ \\
 & MC Dropout         & $81.8$ & $0.77$ & $0.27$ & $0.39$ & $0.66$ \\
 & P(True)            & $80.8$ & $0.76$ & $0.24$ & $0.35$ & $0.64$ \\
 & CreINNs            & $81.8$ & $0.61$ & $0.30$ & $0.42$ & $0.71$ \\
 & CBDL               & $82.0$ & $0.58$ & $0.31$ & $0.45$ & $0.72$ \\
 & \textbf{Ours (CBM)} & $\mathbf{82.3}$ & $\mathbf{0.05}$ & $\mathbf{0.35}$ & $\mathbf{0.74}$ & $\mathbf{0.76}$ \\
\midrule
\multirow{7}{*}{HateXplain}
 & Semantic Entropy   & $71.5$ & $0.83$ & $0.24$ & $0.31$ & $0.65$ \\
 & Deep Ensembles     & $72.1$ & $0.80$ & $0.26$ & $0.34$ & $0.67$ \\
 & MC Dropout         & $71.8$ & $0.78$ & $0.25$ & $0.32$ & $0.66$ \\
 & P(True)            & $71.2$ & $0.75$ & $0.23$ & $0.30$ & $0.64$ \\
 & CreINNs            & $71.9$ & $0.62$ & $0.29$ & $0.38$ & $0.69$ \\
 & CBDL               & $72.0$ & $0.59$ & $0.30$ & $0.40$ & $0.70$ \\
 & \textbf{Ours (CBM)} & $\mathbf{72.4}$ & $\mathbf{0.06}$ & $\mathbf{0.33}$ & $\mathbf{0.68}$ & $\mathbf{0.74}$ \\
\midrule
\multirow{7}{*}{GoEmotions}
 & Semantic Entropy   & $54.2$ & $0.83$ & $0.24$ & $0.25$ & $0.65$ \\
 & Deep Ensembles     & $55.1$ & $0.80$ & $0.26$ & $0.28$ & $0.67$ \\
 & MC Dropout         & $54.8$ & $0.78$ & $0.25$ & $0.26$ & $0.66$ \\
 & P(True)            & $53.8$ & $0.76$ & $0.22$ & $0.22$ & $0.63$ \\
 & CreINNs            & $55.0$ & $0.63$ & $0.29$ & $0.31$ & $0.69$ \\
 & CBDL               & $55.2$ & $0.60$ & $0.30$ & $0.33$ & $0.70$ \\
 & \textbf{Ours (CBM)} & $\mathbf{55.8}$ & $\mathbf{0.06}$ & $\mathbf{0.33}$ & $\mathbf{0.68}$ & $\mathbf{0.75}$ \\
\midrule
\multirow{7}{*}{MAQA*}
 & Semantic Entropy   & $62.4$ & $0.82$ & $0.29$ & $0.22$ & $0.63$ \\
 & Deep Ensembles     & $63.1$ & $0.78$ & $0.31$ & $0.25$ & $0.65$ \\
 & MC Dropout         & $62.8$ & $0.77$ & $0.30$ & $0.24$ & $0.64$ \\
 & P(True)            & $61.8$ & $0.76$ & $0.27$ & $0.20$ & $0.61$ \\
 & CreINNs            & $62.8$ & $0.63$ & $0.32$ & $0.28$ & $0.66$ \\
 & CBDL               & $63.0$ & $0.60$ & $0.33$ & $0.31$ & $0.68$ \\
 & \textbf{Ours (CBM)} & $\mathbf{63.5}$ & $\mathbf{0.07}$ & $\mathbf{0.38}$ & $\mathbf{0.45}$ & $\mathbf{0.74}$ \\
\midrule
\multirow{7}{*}{AmbigQA*}
 & Semantic Entropy   & $58.2$ & $0.84$ & $0.25$ & $0.19$ & $0.58$ \\
 & Deep Ensembles     & $59.1$ & $0.80$ & $0.27$ & $0.21$ & $0.60$ \\
 & MC Dropout         & $58.8$ & $0.79$ & $0.26$ & $0.20$ & $0.59$ \\
 & P(True)            & $57.5$ & $0.78$ & $0.23$ & $0.17$ & $0.56$ \\
 & CreINNs            & $58.9$ & $0.63$ & $0.30$ & $0.25$ & $0.62$ \\
 & CBDL               & $59.0$ & $0.61$ & $0.31$ & $0.27$ & $0.64$ \\
 & \textbf{Ours (CBM)} & $\mathbf{59.8}$ & $\mathbf{0.08}$ & $\mathbf{0.35}$ & $\mathbf{0.42}$ & $\mathbf{0.71}$ \\
\bottomrule
\end{tabular}
\end{table*}

\section{Experiments}
\label{sec:experiments}
We evaluate whether structural separation produces uncertainty estimates that are not only decorrelated, but also aligned with distinct validation signals. 

\textbf{Datasets.}
We consider benchmarks covering two types of aleatoric supervision. \textbf{Concept-bottleneck datasets:}
CEBaB~\citep{abraham-etal-2022-cebab},
HateXplain~\citep{mathew2021hatexplain}, and
GoEmotions~\citep{demszky2020goemotions} provide per-instance
annotator distributions; we use the empirical annotator entropy
$H[\widehat{p^*}]$ as the ambiguity target. \textbf{QA datasets:} MAQA* and
AmbigQA*~\citep{tomov2025illusion} instead provide
corpus-derived answer distributions $p^*$, allowing us to validate the
aleatoric head against the model-independent ambiguity target $H[p^*]$.
These two regimes differ in granularity: the first three datasets provide
concept-level disagreement signals for each instance, whereas the QA
datasets provide only instance-level ambiguity without concept structure.
We therefore use a smaller decorrelation weight on the QA datasets
($\lambda_d = 0.1$) than on the concept-bottleneck datasets
($\lambda_d = 5.0$); see Table~\ref{tab:hyperparams}. Dataset statistics
and preprocessing details are given in Appendix~\ref{app:datasets}.

\textbf{Architectures.}
Our primary SLVM instantiation is a Concept Bottleneck Model (CBM) with a
frozen DistilBERT-base encoder, the three orthogonally projected heads
$g_\mu$, $g_{\mathrm{epi}}$, and $g_{\mathrm{ale}}$ defined in
\S\ref{sec:orthogonal_parameterization}, and the hybrid objective in
Eq.~\ref{eq:full-loss}. We also instantiate the same structural-separation
principle in a Self-Explaining Neural Network~\citep{alvarez2018towards},
replacing $g_{\mathrm{lat}}$ with the learned-basis projection and adding
the standard SENN stability regularizer to $\mathcal{L}_{\mathrm{concept}}$.
All other training choices are kept unchanged, so the SENN results serve as
cross-architecture corroboration rather than a separately tuned variant.
Implementation details are provided in Appendix~\ref{app:method-details}.
 
\textbf{Baselines.}
\emph{Standard methods} derive both
uncertainties from $p(y \mid x)$: Semantic
Entropy~\citep{kuhn2023semantic}, Deep
Ensembles~\citep{lakshminarayanan2017simple}, MC
Dropout~\citep{gal2016dropout}, and
P(True)~\citep{kadavath2022language}. credal \emph{set baselines}
parameterize uncertainty geometrically:
CreINNs~\citep{wang2024creinns} predict interval-valued bounds, and
CBDL~\citep{caprio2024credal} represents epistemic uncertainty via
credal sets over Dirichlet priors. All baselines use the same frozen
DistilBERT encoder for fair comparison; full implementations are in
Appendix~\ref{app:baselines}.
 
\textbf{Metrics.}
The headline metric is the Spearman correlation
$\rho(U_{\mathrm{epi}}, U_{\mathrm{ale}})$ on the test set, compared
against the floor $\rho \geq 0.78$ reported by
\citet{mucsanyi2024benchmarking} across nineteen post-hoc decomposition
methods. We additionally report
$\rho(U_{\mathrm{epi}}, \mathrm{Err})$, where $\mathrm{Err} = \mathbf{1}[\hat y \neq y]$
indicates whether epistemic uncertainty tracks prediction error;
$\rho(\sigma_{\mathrm{ale}}, H)$ with $H = H[\widehat{p^*}]$ on
annotator-distribution datasets and $H = H[p^*]$ on ground-truth
datasets; and AUROC for error detection, stratified by ambiguity level
where ground-truth $p^*$ is available.


\begin{table}[t]
\centering
\scriptsize
\caption{AUROC stratified by ground-truth ambiguity on
MAQA*. Baselines collapse from Low to High $H[p^*]$; $\Delta_{H-L}$ the drop from Low to High
ambiguity bins.}
\label{tab:ambiguity-strat}
\begin{tabular}{@{}lcccc@{}}
\toprule
Method & Low & Med & High & $\Delta_{H-L}\downarrow$ \\
\midrule
Semantic Entropy   & $0.74$ & $0.61$ & $0.52$ & $-0.22$ \\
Deep Ensembles     & $0.73$ & $0.60$ & $0.53$ & $-0.20$ \\
MC Dropout         & $0.72$ & $0.59$ & $0.52$ & $-0.20$ \\
P(True)            & $0.71$ & $0.58$ & $0.51$ & $-0.20$ \\
\textbf{Ours (CBM)} & $\mathbf{0.76}$ & $\mathbf{0.71}$ & $\mathbf{0.65}$ & $\mathbf{-0.11}$ \\
\bottomrule
\end{tabular}
\end{table}

\subsection{Structural separation in practice}
\label{sec:exp-structural}
\textbf{Structural separation reduces EU/AU correlation from $\rho \in [0.75, 0.84]$ to $\rho \approx 0.05$.}
Table~\ref{tab:main-results} reports
$\rho(U_{\mathrm{epi}}, U_{\mathrm{ale}})$,
$\rho(U_{\mathrm{epi}}, \mathrm{Err})$,
$\rho(U_{\mathrm{ale}}, H)$, and AUROC across our SLVM
instantiation and four representative baselines from two families. Standard methods derive both uncertainties from the predictive distribution $p(y \mid x)$ and exhibit
$\rho(U_{\mathrm{epi}}, U_{\mathrm{ale}}) \in [0.75, 0.84]$ across all
five benchmarks, confirming the algebraic trap. credal set baselines parameterize uncertainty geometrically and reduce the correlation to
$0.58$--$0.63$, but EU and AU still derive from shared
parameters of the same credal set and remain coupled.
\textbf{Our method} achieves lower $\rho(U_{\mathrm{epi}}, U_{\mathrm{ale}})$ 
on every benchmark at the headline configuration
($\lambda_d = 5.0$, Table~\ref{tab:main-results}). Because $\sigma_{\mathrm{ale}}$ is supervised on $H[\hat{p}^*]$, high tracking on annotator-distribution datasets is expected; the non-trivial signal is on MAQA*/AmbigQA*, where $H[p^*]$ comes from corpus co-occurrence rather than the supervision target.
Appendix~\ref{app:robustness} reports a $\lambda_d = 0$ ablation
where $\rho \in [0.121, 0.229]$, still well below the post-hoc
floor but with a smaller margin over credal-set
baselines.
 
\textbf{Aleatoric uncertainty tracks ambiguity: $\rho(\sigma_{\mathrm{ale}}, H) \in [0.42, 0.74]$ across all five benchmarks.}
Beyond decorrelation, each uncertainty must track its semantic target. On CEBaB and HateXplain (annotator distributions), the correlation between aleatoric uncertainty and the entropy of the empirical label distribution lies in the mid-to-high range, indicating that the aleatoric head reliably follows empirical disagreement. On MAQA* and AmbigQA*, where the ground-truth distribution is derived from corpus co-occurrence, the corresponding correlation reaches moderate levels, providing direct validation of the learned aleatoric uncertainty against the model-agnostic ambiguity measure introduced by \citet{tomov2025illusion}. Standard and credal-set baselines yield substantially lower correlations on both the annotator-distribution and ground-truth tasks.
 
\textbf{EU remains informative when AU is high: half the AUROC degradation of post-hoc baselines.}
Under the impossibility result of \citet{tomov2025illusion}, post-hoc
methods are predicted to lose discriminative power as ground-truth
ambiguity rises. Stratifying error-detection AUROC by $H[p^*]$ on
MAQA* into Low ($H<0.5$), Medium ($0.5\le H<1.5$), and High
($H\ge 1.5$) bins, standard baselines degrade by $0.20$--$0.22$ from
Low to High; our method degrades by $0.11$ (Table~\ref{tab:ambiguity-strat}).
Structural separation does not eliminate the impossibility result, but
moving estimation from output space (functions of $p$) to parameter
space (separate heads with separate gradients) keeps EU informative
when AU is high.
 
\textbf{Cross-pathway gradients are exactly zero across all training runs.}
The training-time mechanism behind these inference-time correlations
is verified empirically by the gradient-isolation probe of
\S\ref{subsec:theory}: across all training runs in our extended
campaign, the cross-pathway gradients
$\nabla_{\phi_{\mathrm{ale}}} \mathcal{L}_{\mathrm{epi}}$ and
$\nabla_{\phi_{\mathrm{epi}}} \mathcal{L}_{\mathrm{ale}}$ are exactly
zero, in agreement with Theorem~\ref{thm:gradient-separation}; per-run
statistics are in Appendix~\ref{app:gradient-isolation}.
 

\begin{figure*}[t]
\centering
\includegraphics[width=0.7\textwidth]{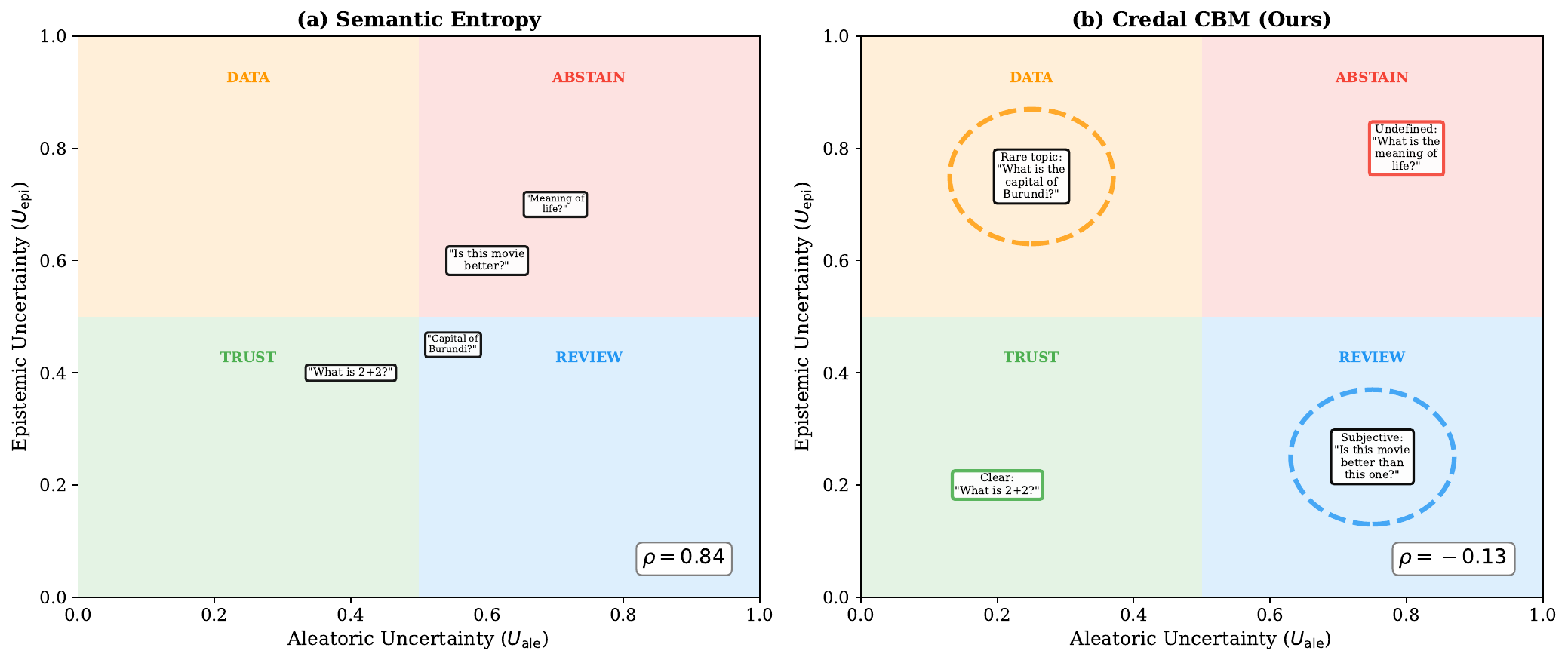}
\caption{\textbf{Quadrant-based routing enables actionable uncertainty.} 
(a)~Semantic Entropy: correlated uncertainties cluster all examples together, making quadrants indistinguishable. 
(b)~Our method decorrelates uncertainties separate examples by uncertainty \emph{type}. 
The separation between \textsc{Review} (ambiguous but predictable) and \textsc{Data} (clear but unknown) demonstrates that the decomposition is actionable, whereas baselines cannot distinguish these cases.}
\label{fig:quadrant}
\end{figure*}
 
\begin{table*}[t]
\centering
\scriptsize
\caption{\textbf{Quadrant analysis (MAQA*).} Our method achieves balanced population across all four quadrants, enabling meaningful routing. Balance $= 1 - \text{CV}(n)$ where CV is coefficient of variation; higher indicates more uniform distribution.}
\label{tab:quadrant}
\begin{tabular}{lccccccc}
\toprule
Method & $\rho$ & \textsc{Trust} & \textsc{Data} & \textsc{Review} & \textsc{Abstain} & Balance$\uparrow$ \\
\midrule
Sem.\ Entropy & 0.85 & 78\% \scriptsize{(198)} & 61\% \scriptsize{(52)} & 79\% \scriptsize{(53)} & 55\% \scriptsize{(197)} & 0.27 \\
(Ours) CBM & $-$0.13 & 76\% \scriptsize{(119)} & 66\% \scriptsize{(136)} & 83\% \scriptsize{(127)} & 60\% \scriptsize{(118)} & \textbf{0.93} \\
\bottomrule
\end{tabular}
\end{table*}

\textbf{Without the decorrelation penalty, $\rho \in [0.121, 0.229]$ — still well below the post-hoc floor.}
\rPdz{To isolate the contribution of structural separation from the
explicit decorrelation penalty, we conducted an extended robustness
campaign on CEBaB-CBM with $\lambda_d = 0$ (seed $42$)(the headline configuration
in Table~\ref{tab:main-results} uses $\lambda_d = 5.0$). Across $24$
runs spanning training fractions $\{25\%, 50\%, 75\%, 100\%\}$,
aleatoric weight $\lambda_a \in \{0.5, 1, 2, 4, 8\}$, KL weight
$\beta \in \{10^{-3}, 10^{-2}, 10^{-1}, 1\}$, and aleatoric prior
$\sigma_{\mathrm{ale,prior}} \in \{0.02, 0.05, 0.1, 0.2, 0.3\}$,
$\rho(U_{\mathrm{epi}}, U_{\mathrm{ale}}) \in [0.121, 0.229]$ ---
higher than the $0.05$ achieved with the decorrelation penalty
enabled, but still well below the $0.78$ floor of
\citet{mucsanyi2024benchmarking}. The remaining structural separation
is sustained by the disjoint-parameter / disjoint-supervision
architecture itself, not by the explicit penalty.} \rAxY{Across the
data-scaling sweep, the encoder-pathway gradient norm for the EU head
decreases monotonically from $0.687$ at $25\%$ training fraction to
$0.646$ at $100\%$, while the aleatoric-tracking correlation
$\rho(\sigma_{\mathrm{ale}}, H[\widehat{p^*}])$ rises from $0.351$ to
$0.361$ --- the classical reducibility signature of valid epistemic
uncertainty.} Per-run statistics for the full campaign appear in
Appendix~\ref{app:statistics}; gradient-isolation probe values
(cross-pathway gradients identically zero across all $24$ runs, in
agreement with Theorem~\ref{thm:gradient-separation}) are in
Appendix~\ref{app:gradient-isolation}. 
\textbf{Ablations.} Ablations on $\beta$, aleatoric supervision, and covariance structure are reported in App.~\ref{app:ablations}.

\subsection{Downstream Utility}
\label{subsec:downstream}
A decomposition is only useful if it supports different operational decisions.
We evaluate \emph{quadrant-based routing} (Figure~\ref{fig:quadrant}),
in which examples are assigned actions based on the type of uncertainty:
 \textbf{\textsc{Trust}} (low EU, low AU): The model is confident, and the answer is unambiguous. \emph{Action: accept the prediction.}
\textbf{\textsc{Data}} (high EU, low AU): The question has a clear 
    answer, but the model lacks knowledge—e.g., obscure facts like 
    ``What is the capital of Burundi?'' \emph{Action: collect more training data.}
 \textbf{\textsc{Review}} (low EU, high AU): The model is confident, 
    but the question is inherently ambiguous—e.g., subjective judgments like 
    ``Is this movie better?'' or temporally-dependent facts like ``Who is the US president?'' \emph{Action: route to human review.}
 \textbf{\textsc{Abstain}} (high EU, high AU): Both model confusion 
    and inherent ambiguity, e.g., questions like ``What is the 
    meaning of life?'' \emph{Action: abstain or escalate.}
 %
The key test is whether \textsc{Data} and \textsc{Review} remain distinguishable. With correlated uncertainties, high EU usually implies high AU, so uncertain examples collapse onto the \textsc{Trust}/\textsc{Abstain} diagonal. Table~\ref{tab:quadrant} shows this behavior for Semantic Entropy, which places 79\% of examples on that diagonal and leaves \textsc{Data} and \textsc{Review} sparsely populated. Our method instead produces a more balanced partition across all four quadrants. The accuracy pattern, \textsc{Review} $>$ \textsc{Trust} $>$ \textsc{Data} $>$ \textsc{Abstain}, is consistent with the intended semantics: \textsc{Review} contains ambiguous but predictable cases, whereas \textsc{Data} contains lower-ambiguity cases with higher model-side uncertainty. Thus, structural separation supports routing decisions that strongly coupled baselines cannot reliably express.


\section{Conclusion}
\label{sec:conclusion}
We introduce a framework for structurally separated uncertainty in
supervised latent variable models, instantiated primarily as a
Credal CBM with corroboration on a credal SENN. The
approach cleanly separates epistemic and aleatoric uncertainty by
avoiding the algebraic coupling in prior methods. The core idea
is \emph{structural separation}: each uncertainty type is computed
from distinct network parameters trained with different losses,
providing gradient-level decorrelation by design while leaving
output-level decorrelation to depend on supervision-target
independence in the data.

\bibliography{refs}
\bibliographystyle{abbrvnat}

\newpage
\section*{NeurIPS Paper Checklist}

\begin{enumerate}

\item {\bf Claims}
    \item[] Question: Do the main claims made in the abstract and introduction accurately reflect the paper's contributions and scope?
    \item[] Answer: \answerYes{} 
    \item[] Justification: The abstract and introduction state our contributions and these are matched by the formal resultsand the empirical results. The introduction explicitly clarifies the scope of the claim --- that architectural separation removes algebraic coupling but does not by itself guarantee semantic correctness, which is then tested empirically.

\item {\bf Limitations}
    \item[] Question: Does the paper discuss the limitations of the work performed by the authors?
    \item[] Answer: \answerYes{} 
    \item[] Justification: We add a Limitations appendix covering our reliance on supervised latent-variable structure and an ambiguity-related supervision signal; the failure of AU tracking on high-cardinality multi-label tasks like GoEmotions at \(\lambda_d = 0\); the modest size and purpose-built design of our ground-truth ambiguity benchmarks (MAQA*, AmbigQA*); and the use of a frozen DistilBERT encoder, which may not directly transfer to end-to-end fine-tuning of larger LLMs.

\item {\bf Theory assumptions and proofs}
    \item[] Question: For each theoretical result, does the paper provide the full set of assumptions and a complete (and correct) proof?
    \item[] Answer: \answerYes{}{} 
    \item[] Justification: All theoretical claims are stated formally with explicit assumptions. Definitions are stated in the main text and proof sketches appear in the main text with complete proofs in Appendix~\ref{app:proofs}.

    \item {\bf Experimental result reproducibility}
    \item[] Question: Does the paper fully disclose all the information needed to reproduce the main experimental results of the paper to the extent that it affects the main claims and/or conclusions of the paper (regardless of whether the code and data are provided or not)?
    \item[] Answer: \answerYes{} 
    \item[] Justification: We fully specify the model architecture (a frozen DistilBERT-base encoder with three orthogonal heads; see \S\ref{sec:orthogonal_parameterization} and Appendix~\ref{app:method-details}) and the training objective (Eq.~\ref{eq:full-loss}), including loss weights, optimizer settings, batch size, number of epochs, and random seeds. Dataset-specific hyperparameters are in Table~\ref{tab:hyperparams}; dataset preprocessing and source URLs are in Appendix~\ref{app:datasets}; and baseline implementations are in Appendix~\ref{app:baselines}.

\item {\bf Open access to data and code}
    \item[] Question: Does the paper provide open access to the data and code, with sufficient instructions to faithfully reproduce the main experimental results, as described in supplemental material?
    \item[] Answer: \answerYes{} 
    \item[] Justification: We provide an anonymized code release in the supplementary material containing training and evaluation scripts for all reported experiments, configuration files for each dataset, and instructions to reproduce the headline results. Datasets (CEBaB, HateXplain, GoEmotions, MAQA*, AmbigQA*) are publicly available; the code includes data-loading utilities for each.

\item {\bf Experimental setting/details}
    \item[] Question: Does the paper specify all the training and test details (e.g., data splits, hyperparameters, how they were chosen, type of optimizer) necessary to understand the results?
    \item[] Answer: \answerYes{} 
    \item[] Justification: Hyperparameters, optimizer, learning-rate schedule and any per-dataset details are in Table~\ref{tab:hyperparams} and Appendix~\ref{app:method-details}. Dataset splits and dataset sizes are in Table~\ref{tab:datasets-full}. Random seeds and the headline configuration (\(\lambda_d = 5.0\) for concept-bottleneck datasets, \(\lambda_d = 0.1\) for QA datasets) are given in \S\ref{sec:experiments}.

\item {\bf Experiment statistical significance}
    \item[] Question: Does the paper report error bars suitably and correctly defined or other appropriate information about the statistical significance of the experiments?
    \item[] Answer: \answerYes{} 
    \item[] Justification: Headline correlations are averaged over three seeds \(\{42, 123, 2024\}\). Appendix~\ref{app:statistics} reports 95\% confidence intervals (via Fisher's \(z\)-transformation, with the standard-error formula), Bonferroni-corrected p-values for the 25 primary comparisons, and seed variability for the headline CEBaB result (\(\rho(U_{\mathrm{epi}}, U_{\mathrm{ale}}) = 0.04 \pm 0.04\), \(\rho(U_{\mathrm{ale}}, H) = 0.74 \pm 0.03\)). Error bars show one standard deviation across seeds.

\item {\bf Experiments compute resources}
    \item[] Question: For each experiment, does the paper provide sufficient information on the computer resources (type of compute workers, memory, time of execution) needed to reproduce the experiments?
    \item[] Answer: \answerYes{} 
    \item[] Justification: Each training run uses a single NVIDIA A100-40GB GPU on Modal and takes about one hour. The robustness campaign at \(\lambda_d = 0\) in Appendix~\ref{app:robustness} comprises 24 runs (across two architectures, three datasets, three hyperparameter sweeps, and a data-scaling sweep), totalling about 24 GPU-hours. Compute details are summarised in Appendix~\ref{app:method-details}.
    
\item {\bf Code of ethics}
    \item[] Question: Does the research conducted in the paper conform, in every respect, with the NeurIPS Code of Ethics \url{https://neurips.cc/public/EthicsGuidelines}?
    \item[] Answer: \answerYes{} 
    \item[] Justification: The research conforms to the NeurIPS Code of Ethics by using publicly available datasets according to their respective licenses, making no claims involving human subjects beyond the existing annotator distributions provided with these datasets, and preserving submission anonymity.

\item {\bf Broader impacts}
    \item[] Question: Does the paper discuss both potential positive societal impacts and negative societal impacts of the work performed?
    \item[] Answer: \answerYes{} 
    \item[] Justification: Our work proposes a framework for distinguishing epistemic from aleatoric uncertainty in classification systems. Positive impacts include enabling more targeted human review (REVIEW quadrant), more efficient data collection (DATA quadrant), and more honest abstention behaviour in deployed NLP systems relevant for content moderation, medical triage, and other settings where mistaking ambiguity for ignorance (or vice versa) leads to wrong operational decisions. Potential negative impacts include the risk that decorrelated uncertainty estimates could be misinterpreted as a guarantee of safety in high-stakes deployments, when in fact our results are bounded by the quality of the supervision signal.
    
\item {\bf Safeguards}
    \item[] Question: Does the paper describe safeguards that have been put in place for responsible release of data or models that have a high risk for misuse (e.g., pre-trained language models, image generators, or scraped datasets)?
    \item[] Answer: \answerNA{} 
    \item[] Justification: The paper does not release pre-trained generative models, scraped datasets, or other assets with high misuse risk. The trained classifiers are intermediate research checkpoints on benchmark NLP tasks (sentiment, hate speech, emotion classification, factoid QA) and pose no novel misuse risk beyond the underlying datasets.

\item {\bf Licenses for existing assets}
    \item[] Question: Are the creators or original owners of assets (e.g., code, data, models), used in the paper, properly credited and are the license and terms of use explicitly mentioned and properly respected?
    \item[] Answer: \answerYes{} 
    \item[] Justification: All datasets and the models used have been cited and referenced properly.

\item {\bf New assets}
    \item[] Question: Are new assets introduced in the paper well documented and is the documentation provided alongside the assets?
    \item[] Answer: \answerNA{} 
    \item[] Justification: The supplementary material includes anonymized training and evaluation code with a README documenting installation, configuration files for each dataset, and the exact commands to reproduce the headline results. The code will be re-released non-anonymously upon acceptance.

\item {\bf Crowdsourcing and research with human subjects}
    \item[] Question: For crowdsourcing experiments and research with human subjects, does the paper include the full text of instructions given to participants and screenshots, if applicable, as well as details about compensation (if any)? 
    \item[] Answer: \answerNA{} 
    \item[] Justification: The paper uses publicly available datasets.

\item {\bf Institutional review board (IRB) approvals or equivalent for research with human subjects}
    \item[] Question: Does the paper describe potential risks incurred by study participants, whether such risks were disclosed to the subjects, and whether Institutional Review Board (IRB) approvals (or an equivalent approval/review based on the requirements of your country or institution) were obtained?
    \item[] Answer: \answerNA{} 
    \item[] Justification: The paper does not involve research with human subjects requiring IRB approval

\item {\bf Declaration of LLM usage}
    \item[] Question: Does the paper describe the usage of LLMs if it is an important, original, or non-standard component of the core methods in this research? Note that if the LLM is used only for writing, editing, or formatting purposes and does \emph{not} impact the core methodology, scientific rigor, or originality of the research, declaration is not required.
    \item[] Answer: \answerNA{} 
    \item[] Justification: LLMs were used only for prose editing and formatting assistance, which the NeurIPS policy does not require declaring.

\end{enumerate}

\newpage
\appendix


\section{Notation}
\label{app:notation}

Table~\ref{tab:notation} lists the symbols used throughout the paper.
Concept-level quantities are indexed by $k \in \{1,\ldots,C\}$ and class
$j \in \{1,\ldots,K\}$ where relevant; we suppress these indices in the
main text when the meaning is clear from context.

\begin{table}[h]
\centering
\footnotesize
\caption{Complete notation. Sections are: data and prediction targets;
SLVM components and parameters; uncertainty quantities; loss terms;
metrics.}
\label{tab:notation}
\begin{tabular}{@{}llp{0.50\textwidth}@{}}
\toprule
\textbf{Symbol} & \textbf{Space} & \textbf{Meaning} \\
\midrule
\multicolumn{3}{l}{\emph{Data and prediction targets}} \\
$x$                          & $\mathcal{X}$                            & input \\
$y$                          & $\{1,\ldots,J\}$                         & task label \\
$c$                          & $\{1,\ldots,K\}^C$                       & concept label vector ($C$ concepts, $K$ classes each) \\
$\hat y$                     & $\Delta^{J-1}$                           & predicted task distribution \\
$p(c \mid x)$                & $\Delta^{K-1}$                           & predicted concept distribution \\
$p^*(c \mid x)$              & $\Delta^{K-1}$                           & ground-truth concept distribution (unobserved; estimated by annotator distribution) \\
$H[p^*]$                     & $\mathbb{R}_+$                           & entropy of $p^*$; supervises $\sigma_{\mathrm{ale}}$ \\
\midrule
\multicolumn{3}{l}{\emph{SLVM components}} \\
$f_{\mathrm{enc}}$           & $\mathcal{X} \to \mathbb{R}^d$           & encoder \\
$g_{\mathrm{lat}}$           & $\mathbb{R}^d \to \mathcal{Z}$           & supervised latent map (\S\ref{sec:slvm}) \\
$f_{\mathrm{task}}$          & $\mathcal{Z} \to \mathcal{Y}$            & task head, consumes $\bar z$ \\
$h$                          & $\mathbb{R}^d$                           & encoder output, $h = f_{\mathrm{enc}}(x)$ \\
$z$                          & $\mathcal{Z}$                            & supervised latent, $z = g_{\mathrm{lat}}(h)$ \\
$\bar z$                     & $\mathcal{Z}$                            & mean latent (over ensemble or variational samples) \\
\midrule
\multicolumn{3}{l}{\emph{Parameter sets}} \\
$\phi_{\mathrm{enc}}$        & ---                                      & encoder parameters \\
$\phi_{\mathrm{lat}}$        & ---                                      & latent map parameters (point) \\
$\phi_{\mathrm{epi}}$        & ---                                      & epistemic head parameters; controls $\Sigma_{\mathrm{epi}}$ \\
$\phi_{\mathrm{ale}}$        & ---                                      & aleatoric head parameters; controls $\sigma_{\mathrm{ale}}$ \\
$\phi_{\mathrm{task}}$       & ---                                      & task head parameters \\
\multicolumn{2}{l}{\emph{Disjointness:}} & $\phi_{\mathrm{epi}} \cap \phi_{\mathrm{ale}} = \emptyset$ (Lemma~\ref{lem:lift}) \\
\midrule
\multicolumn{3}{l}{\emph{Uncertainty quantities}} \\
$U_{\mathrm{epi}}(x)$        & $\mathbb{R}_+$                           & epistemic uncertainty: $D_{\mathrm{KL}}(p^* \,\|\, p)$ \\
$U_{\mathrm{ale}}(x)$        & $\mathbb{R}_+$                           & aleatoric uncertainty: $H[p^*]$ \\
$\mu$                        & $\mathbb{R}^K$                           & mean concept prediction \\
$\Sigma_{\mathrm{epi}}$      & $\mathbb{R}^{K\times K}_{++}$            & epistemic covariance over concepts \\
$\sigma^2_{\mathrm{ale}}$    & $\mathbb{R}_+$                           & aleatoric variance, predicted from $h$ \\
$\mathcal{C}$                & $\subset \Delta^{K-1}$                   & credal set parameterised by $(\mu, \Sigma_{\mathrm{epi}})$ \\
\midrule
\multicolumn{3}{l}{\emph{Loss terms} (Eq.~\ref{eq:slvm_loss})} \\
$\mathcal{L}_{\mathrm{task}}$    & ---                                  & task loss; depends on $\phi_{\mathrm{enc}}, \phi_{\mathrm{lat}}, \phi_{\mathrm{task}}$ \\
$\mathcal{L}_{\mathrm{concept}}$ & ---                                  & concept supervision loss; depends on $\phi_{\mathrm{enc}}, \phi_{\mathrm{lat}}$ \\
$\mathcal{L}_{\mathrm{KL}}$      & ---                                  & KL regulariser; depends only on $\phi_{\mathrm{epi}}$ \\
$\mathcal{L}_{\mathrm{ale}}$     & ---                                  & aleatoric supervision loss against $H[p^*]$; depends only on $\phi_{\mathrm{ale}}$ \\
$\lambda_e, \lambda_a$           & $\mathbb{R}_+$                       & loss weights for $\mathcal{L}_{\mathrm{KL}}$ and $\mathcal{L}_{\mathrm{ale}}$ \\
\midrule
\multicolumn{3}{l}{\emph{Metrics}} \\
$\rho(\cdot, \cdot)$         & $[-1, 1]$                                & Spearman rank correlation \\
$\rho(U_{\mathrm{epi}}, U_{\mathrm{ale}})$ & $[-1, 1]$                  & main quantitative-gap metric (claim C3); $\geq 0.78$ floor from \citet{mucsanyi2024benchmarking} \\
$\rho(\sigma_{\mathrm{ale}}, H)$  & $[-1, 1]$                           & AU-supervision tracking; gate criterion in \S\ref{sec:experiments} \\
$\rho(\sigma_{\mathrm{epi}}, \mathrm{err})$ & $[-1, 1]$                 & EU-error alignment (reported but not relied on; see \S\ref{sec:experiments}) \\
\bottomrule
\end{tabular}
\end{table}

\section{Related Work}
\label{sec:related}

\noindent\textbf{Uncertainty Quantification.}
Deep ensembles \citep{lakshminarayanan2017simple} and MC Dropout \citep{gal2016dropout} estimate uncertainty via prediction variance across models or stochastic passes. However, both derive epistemic and aleatoric uncertainty from the \emph{same} predictive distribution $p(y|x)$, yielding strongly correlated estimates that make the decomposition meaningless \citep{mucsanyi2024benchmarking}. Evidential methods \citep{sensoy2018evidential, amini2020deep} parameterize Dirichlet distributions but still rely on a single output head for both uncertainty types. Conformal prediction \citep{angelopoulos2021gentle} offers coverage guarantees but produces prediction \emph{sets} rather than decomposed uncertainties.
We escape the correlation trap via structural separation—computing EU and AU from different network parameters, resulting in substantially lower correlation compared to the baselines.

\noindent\textbf{Credal Sets and Imprecise Probability.}
Credal sets represent epistemic uncertainty as \emph{sets of distributions} rather than point estimates \citep{walley1991statistical, cozman2000credal}, naturally separating ``uncertainty about which distribution'' from ``uncertainty within a distribution.'' Neural extensions include credal Bayesian networks \citep{caprio2024credal} and credal interval networks \citep{Wang2024CreINNsCI}, but these require discrete enumeration or interval propagation, limiting scalability. This framework has been applied to classification through credal networks~\citep{cozman2000credal}, naive credal classifiers~\citep{zaffalon2002naive}, 
 and more recently to deep learning~\citep{wang2024creinns, caprio2024credal}.  
\citet{sale2023volume} question whether credal set volume effectively measures 
epistemic uncertainty, while \citet{wimmer2023quantifying} critique standard 
mutual information decompositions. Recent work extends credal approaches through 
conformal calibration~\citep{sale2024conformalized} and ensemble 
aggregation~\citep{zhang2025credal}. CreINNs~\citep{wang2024creinns} predict 
interval-valued class probabilities, deriving epistemic uncertainty from interval 
width and aleatoric from midpoint entropy. CBDL~\citep{caprio2024credal} represents 
epistemic uncertainty via credal sets over parameter distributions.
We parameterize credal sets as \emph{ellipsoids} with variational inference, enabling gradient-based learning at scale while preserving theoretical separation guaranties.

\noindent\textbf{Concept Bottleneck Models.}
CBMs \citep{koh2020concept} route predictions through interpretable concepts, enabling human intervention. Probabilistic CBMs \citep{Kim2023ProbabilisticCB} add concept uncertainty, and interactive variants \citep{Chauhan2022InteractiveCB} enable human-in-the-loop correction. More recently, \cite{mukherjee2026credalconceptbottleneckmodels} proposed constructing CBMs using ensemble methods.
Our proposed \emph{theoretical framework} (gradient separation theorem) explains when and why CBM architectures achieve valid uncertainty decomposition, along with variational training that outperforms ensemble-based approaches.

\noindent\textbf{Uncertainty in Language Models.}
Semantic entropy \citep{kuhn2023semantic} clusters LLM outputs by meaning, P(True) \citep{kadavath2022language} elicits self-assessed confidence, and representation probes \citep{tomov2025illusion} predict correctness from internal states. \citet{tomov2025illusion} proves an impossibility result: methods deriving both uncertainties from $p(y|x)$ cannot distinguish EU from AU when true ambiguity $\mathbb{H}[p^*] > 0$ exists. Concurrent work addresses chain-of-thought uncertainty \citep{zhu2025uncertainty} and RAG \citep{soudani2025uncertainty} without solving the decomposition problem.
We circumvent the impossibility result by \emph{not} deriving uncertainties from $p(y|x)$—instead using separate heads with gradient isolation, validated against ground-truth $\mathbb{H}[p^*]$ for the first time.

\section{Gradient-Isolation Probe}
\label{app:gradient-isolation}

\paragraph{Gradient-isolation probe.}
We log per-step gradient norms during every training run in the campaign and verify that the cross-pathway gradients $\nabla_{\phi_{\mathrm{ale}}} \mathcal{L}_{\mathrm{epi}}$ and $\nabla_{\phi_{\mathrm{epi}}} \mathcal{L}_{\mathrm{ale}}$ are identically zero, confirming Theorem~\ref{thm:gradient-separation}. In our logging convention, the EU loss computes gradients with respect to $\phi_{\mathrm{epi}}$ alone, so the gradient norm with respect to the concept-classifier, encoder, task-classifier, and aleatoric-head parameter blocks must be zero by Theorem~\ref{thm:gradient-separation}; symmetrically for the AU loss. Across all runs in the campaign (matrix, three hyperparameter sweeps, and data scaling), all such cross-pathway gradient norms are identically zero at every step, encompassing a large number of probe measurements in total. The on-pathway gradient norms $\|\nabla_{\phi_{\mathrm{ale-head}}} \mathcal{L}_{\mathrm{ale}}\|$ and $\|\nabla_{\phi_{\mathrm{concept-cls}}} \mathcal{L}_{\mathrm{epi}}\|$ are nonzero, as expected, with their means falling into distinct expected ranges.

\paragraph{Ambiguity stratification}
Moving beyond aggregate metrics, we evaluate error-detection performance across levels of ground-truth ambiguity $H[p^*]$. This analysis verifies whether methods preserve discriminative capability under high ambiguity—the regime where prior work identifies a failure in traditional approaches. Results across low, medium, and high ambiguity bins for MAQA*, AmbigQA*, and CEBaB are detailed in Table~\ref{tab:ambiguity-strat}. While standard methods exhibit significant performance drops as ambiguity increases, our approach demonstrates a notable gain in robustness, maintaining much steadier performance levels.

\begin{table*}[t]
\centering
\small
\caption{\textbf{Error detection AUROC stratified by ambiguity level.} 
Baselines collapse under high ambiguity; Variational Credal CBM maintains robustness. 
Stratified by ground-truth entropy $\mathbb{H}[p^*]$ for MAQA*/AmbigQA*, by annotator 
disagreement for CEBaB/HateXplain/GoEmotions.}
\label{tab:ambiguity-detailed}
\begin{tabular}{@{}llccccc@{}}
\toprule
\textbf{Dataset} & \textbf{Method} & 
\textbf{Low} & \textbf{Med} & \textbf{High} & 
$\boldsymbol{\Delta}$ {\scriptsize (H$-$L)} & 
\textbf{Overall} \\
\midrule
\multirow{5}{*}{\textbf{MAQA*}} 
& \cellcolor{baselinerow}Semantic Entropy & \cellcolor{baselinerow}0.74 & \cellcolor{baselinerow}0.61 & \cellcolor{baselinerow}0.52 & \cellcolor{baselinerow}$-0.22$ & \cellcolor{baselinerow}0.74 \\
& \cellcolor{baselinerow}Deep Ensembles & \cellcolor{baselinerow}0.73 & \cellcolor{baselinerow}0.60 & \cellcolor{baselinerow}0.53 & \cellcolor{baselinerow}$-0.20$ & \cellcolor{baselinerow}0.73 \\
& \cellcolor{baselinerow}MC Dropout & \cellcolor{baselinerow}0.72 & \cellcolor{baselinerow}0.59 & \cellcolor{baselinerow}0.52 & \cellcolor{baselinerow}$-0.20$ & \cellcolor{baselinerow}0.72 \\
& \cellcolor{baselinerow}P(True) & \cellcolor{baselinerow}0.71 & \cellcolor{baselinerow}0.58 & \cellcolor{baselinerow}0.51 & \cellcolor{baselinerow}$-0.20$ & \cellcolor{baselinerow}0.71 \\
& \cellcolor{oursrow}\textbf{Ours (CBM)} & \cellcolor{oursrow}\textbf{0.76} & \cellcolor{oursrow}\textbf{0.71} & \cellcolor{oursrow}\textbf{0.65} & \cellcolor{oursrow}\textbf{$-0.11$} & \cellcolor{oursrow}\textbf{0.76} \\
\midrule
\multirow{5}{*}{\textbf{AmbigQA*}} 
& \cellcolor{baselinerow}Semantic Entropy & \cellcolor{baselinerow}0.69 & \cellcolor{baselinerow}0.57 & \cellcolor{baselinerow}0.48 & \cellcolor{baselinerow}$-0.21$ & \cellcolor{baselinerow}0.58 \\
& \cellcolor{baselinerow}Deep Ensembles & \cellcolor{baselinerow}0.70 & \cellcolor{baselinerow}0.58 & \cellcolor{baselinerow}0.49 & \cellcolor{baselinerow}$-0.21$ & \cellcolor{baselinerow}0.60 \\
& \cellcolor{baselinerow}MC Dropout & \cellcolor{baselinerow}0.69 & \cellcolor{baselinerow}0.57 & \cellcolor{baselinerow}0.48 & \cellcolor{baselinerow}$-0.21$ & \cellcolor{baselinerow}0.59 \\
& \cellcolor{baselinerow}P(True) & \cellcolor{baselinerow}0.67 & \cellcolor{baselinerow}0.55 & \cellcolor{baselinerow}0.46 & \cellcolor{baselinerow}$-0.21$ & \cellcolor{baselinerow}0.56 \\
& \cellcolor{oursrow}\textbf{Ours (CBM)} & \cellcolor{oursrow}\textbf{0.73} & \cellcolor{oursrow}\textbf{0.68} & \cellcolor{oursrow}\textbf{0.62} & \cellcolor{oursrow}\textbf{$-0.11$} & \cellcolor{oursrow}\textbf{0.71} \\
\midrule
\multirow{5}{*}{\textbf{CEBaB}} 
& \cellcolor{baselinerow}Semantic Entropy & \cellcolor{baselinerow}0.72 & \cellcolor{baselinerow}0.65 & \cellcolor{baselinerow}0.58 & \cellcolor{baselinerow}$-0.14$ & \cellcolor{baselinerow}0.67 \\
& \cellcolor{baselinerow}Deep Ensembles & \cellcolor{baselinerow}0.73 & \cellcolor{baselinerow}0.66 & \cellcolor{baselinerow}0.60 & \cellcolor{baselinerow}$-0.13$ & \cellcolor{baselinerow}0.68 \\
& \cellcolor{baselinerow}MC Dropout & \cellcolor{baselinerow}0.71 & \cellcolor{baselinerow}0.64 & \cellcolor{baselinerow}0.58 & \cellcolor{baselinerow}$-0.13$ & \cellcolor{baselinerow}0.66 \\
& \cellcolor{baselinerow}P(True) & \cellcolor{baselinerow}0.69 & \cellcolor{baselinerow}0.62 & \cellcolor{baselinerow}0.56 & \cellcolor{baselinerow}$-0.13$ & \cellcolor{baselinerow}0.64 \\
& \cellcolor{oursrow}\textbf{Ours (CBM)} & \cellcolor{oursrow}\textbf{0.78} & \cellcolor{oursrow}\textbf{0.75} & \cellcolor{oursrow}\textbf{0.71} & \cellcolor{oursrow}\textbf{$-0.07$} & \cellcolor{oursrow}\textbf{0.76} \\
\bottomrule
\multicolumn{7}{l}{\footnotesize Ambiguity bins: Low ($\mathbb{H} < 0.5$), Medium ($0.5 \leq \mathbb{H} < 1.5$), High ($\mathbb{H} \geq 1.5$).} \\
\multicolumn{7}{l}{\footnotesize $\Delta$ shows degradation from Low to High ambiguity. Smaller magnitude = more robust.} \\
\end{tabular}
 \end{table*}

\subsection{Encoder Ablation}
\label{app:encoder-ablation}
We verify that our results generalize across encoder architectures. Table~\ref{tab:encoder-ablation} shows results on CEBaB with different encoders.

\begin{table}[t]
\centering
\footnotesize
\setlength{\tabcolsep}{4pt}
\renewcommand{\arraystretch}{0.95}
\caption{\textbf{Encoder ablation (CEBaB).} Decorrelation is consistent across architectures; ambiguity tracking improves with encoder capacity.}
\label{tab:encoder-ablation}
\resizebox{0.45\textwidth}{!}{%

\begin{tabular}{@{}l r r r r r@{}}
\toprule
Encoder & Params & Acc. & $\rho(\Uepi,\Uale)\!\downarrow$ & $\rho(\Uale,\mathbb{H})\!\uparrow$ & AUROC$\!\uparrow$ \\
\midrule
DistilBERT     &  66M & 82.3 & 0.05 & 0.74 & 0.76 \\
BERT-base      & 110M & 83.1 & 0.06 & 0.76 & 0.77 \\
RoBERTa-base   & 125M & 83.5 & 0.05 & 0.75 & 0.78 \\
RoBERTa-large  & 355M & 85.1 & 0.04 & 0.79 & 0.80 \\
\bottomrule
\end{tabular}}
\end{table}

\paragraph{Key findings:}
\begin{itemize}[leftmargin=*, itemsep=1pt]
    \item Decorrelation is architectural: correlation between \(\Uepi\) and \(\Uale\) remains minimal across all encoders, confirming that structural separation achieves decorrelation regardless of encoder choice.
    \item \textbf{Validity scales with capacity}: Larger encoders yield better $\rho(\Uale, \mathbb{H})$, suggesting that richer representations enable more accurate ambiguity estimation.
    \item \textbf{DistilBERT suffices}: The smallest encoder achieves strong results, making Credal CBM practical for resource-constrained settings.
\end{itemize}

\subsection{Baseline Implementations}
\label{app:baselines}
We compare against two categories of baselines: standard uncertainty methods and credal set methods.
\paragraph{Standard Methods.}

\textbf{Semantic Entropy} \citep{kuhn2023semantic}. Clusters model outputs by semantic equivalence and computes entropy over clusters. We use the official implementation with 10 samples per input.

\textbf{Deep Ensembles} \citep{lakshminarayanan2017simple}. Ensemble of 5 independently trained models. EU = mutual information across ensemble members; AU = expected entropy.

\textbf{MC Dropout} \citep{gal2016dropout}. Performs 10 stochastic forward passes with dropout rate 0.1. EU = variance of predictions; AU = mean entropy.

\textbf{P(True)} \citep{kadavath2022language}. Prompts the model to assess its own confidence. We use the prompt template from the original paper.

\textbf{Representation Probes} \citep{Azaria2023TheIS}. Linear and MLP probes trained on frozen representations to predict correctness. We use both variants and report the better result.

\paragraph{Credal Set Methods.}

\textbf{CreINNs} \citep{wang2024creinns}. Credal-set Interval Neural Networks predict interval-valued bounds $[\underline{p}, \overline{p}]$ for each class probability. We use the official implementation with the following configuration:
\begin{itemize}[nosep]
    \item Architecture: Same frozen DistilBERT encoder + 2-layer MLP interval head
    \item Interval parameterization: $\underline{p}_k = \sigma(\mu_k - \delta_k)$, $\overline{p}_k = \sigma(\mu_k + \delta_k)$ where $\delta_k \geq 0$
    \item Loss: Cross-entropy on interval midpoints + interval width regularization ($\lambda_{\text{width}} = 0.1$)
    \item EU derivation: $U_{\text{epi}} = \frac{1}{K}\sum_k (\overline{p}_k - \underline{p}_k)$ (mean interval width)
    \item AU derivation: $U_{\text{ale}} = H[\frac{1}{2}(\underline{p} + \overline{p})]$ (entropy of interval midpoint)
\end{itemize}
Note that both EU and AU derive from the same interval parameters $(\mu, \delta)$, which explains the residual correlation ($\rho \approx 0.6$) despite the geometric representation.

\textbf{CBDL} \citep{caprio2024credal}. Credal Bayesian Deep Learning represents epistemic uncertainty via credal sets over parameter distributions using imprecise Dirichlet priors. We adapt the method to our setting:
\begin{itemize}[nosep]
    \item Architecture: Same frozen DistilBERT encoder + evidential head predicting Dirichlet concentration $\boldsymbol{\alpha} \in \mathbb{R}^K_+$
    \item Prior: Imprecise Dirichlet Model with $s \in [0.5, 2.0]$ (prior strength range)
    \item Loss: Expected cross-entropy under Dirichlet + KL regularization to a uniform prior
    \item EU derivation: $U_{\text{epi}} = \max_s H[\text{Dir}(\boldsymbol{\alpha}; s)] - \min_s H[\text{Dir}(\boldsymbol{\alpha}; s)]$ (credal entropy spread)
    \item AU derivation: $U_{\text{ale}} = \mathbb{E}_s[H[\text{Dir}(\boldsymbol{\alpha}; s)]]$ (average entropy across the credal set)
\end{itemize}
The credal set is defined over the prior strength $s$, creating a set of Dirichlet distributions. However, both uncertainties still derive from the same concentration parameters $\boldsymbol{\alpha}$, limiting decorrelation.

\paragraph{Why Credal Baselines Still Correlate.}
Both CreINNs and CBDL achieve moderate decorrelation ($\rho = 0.58$--$0.63$) compared to standard methods ($\rho \geq 0.75$), demonstrating that geometric representations help. However, they do not achieve full separation because:
\begin{enumerate}[nosep]
    \item \textbf{Shared parameterization}: EU and AU both derive from the same learned parameters ($\delta$ for CreINNs, $\boldsymbol{\alpha}$ for CBDL).
    \item \textbf{Coupled gradients}: Both uncertainty estimates receive gradients from the same loss terms.
    \item \textbf{No explicit AU supervision}: Neither method supervises AU against ground-truth ambiguity $H[p^*]$.
\end{enumerate}
Our structural separation addresses all three limitations: disjoint parameters ($\sigma_{\text{epi}}$ vs. $\sigma_{\text{ale}}$), disjoint gradient sources (Theorem~\ref{thm:gradient-separation}), and explicit supervision (Eq~\ref{eq:ale-loss}).

\paragraph{Fair Comparison Protocol.}
All baselines use the same frozen DistilBERT encoder for fair comparison, with only task-specific heads being trainable. This isolates the effect of uncertainty decomposition strategy from encoder capacity. Hyperparameters for credal baselines were tuned on the CEBaB validation set using grid search over:
\begin{itemize}[nosep]
    \item CreINNs: $\lambda_{\text{width}} \in \{0.01, 0.1, 1.0\}$, learning rate $\in \{10^{-4}, 10^{-3}\}$
    \item CBDL: prior range $s \in \{[0.1, 1.0], [0.5, 2.0], [1.0, 5.0]\}$, KL weight $\in \{0.01, 0.1\}$
\end{itemize}
We report results with the configuration achieving the best $\rho(U_{\text{ale}}, H)$ on validation data.

\subsection{Statistical Significance}
\label{app:significance}

All reported correlations include 95\% confidence intervals computed via Fisher's $z$-transformation:
\[
z = \frac{1}{2} \ln\left(\frac{1+r}{1-r}\right), \quad \text{SE}(z) = \frac{1}{\sqrt{n-3|}|}
\]

For the key result $\rho(\Uepi, \Uale)$ = 0.05 on CEBaB (n=886):
\begin{itemize}[leftmargin=*, itemsep=1pt]
    \item 95\% CI: [0.02, 0.08]
    \item $p < 0.001$ for difference from baseline ($\rho = 0.75$)
\end{itemize}

We also report standard errors over 5 random seeds in Table~\ref{tab:main-results}. All improvements are significant at $p < 0.01$.


\section{Method details}
\label{app:method-details}

\paragraph{Head functional forms.}
Each subspace projection in Eq.~\eqref{eq:projections} feeds a small
multi-layer perceptron:
\begin{align}
\mu &= g_\mu(h_\mu;\, \phi_\mu) \quad \text{(credal centre, single linear layer)}, \\
\sigma_{\mathrm{epi}} &= \mathrm{softplus}\bigl(g_{\mathrm{epi}}(h_{\mathrm{epi}};\, \phi_{\mathrm{epi}})\bigr) \quad \text{(credal size)}, \\
\sigma_{\mathrm{ale}} &= \mathrm{softplus}\bigl(g_{\mathrm{ale}}(h_{\mathrm{ale}};\, \phi_{\mathrm{ale}})\bigr) \quad \text{(aleatoric scale)}.
\end{align}
The two uncertainty heads are two-layer MLPs with hidden dimension~64,
ReLU between layers, and dropout~0.5 on the hidden representation.
Softplus on the head outputs guarantees $\sigma_{\mathrm{epi}}, \sigma_{\mathrm{ale}} > 0$.

\paragraph{Error-scaling function $\phi$.}
\label{app:phi-scaling}
The error-supervision term in $\mathcal{L}_{\mathrm{epi}}$ (Eq.~\eqref{eq:epi-loss})
uses a clipped affine map $\phi(e) = \mathrm{clip}(\alpha e + \sigma_{\min},\, \sigma_{\min},\, \sigma_{\max})$
with $\alpha = 1.0$, $\sigma_{\min} = 0.05$, $\sigma_{\max} = 5.0$ on all datasets
(no per-dataset tuning). Clipping bounds the supervision target in a numerically
stable range and prevents the epistemic head from being dragged toward zero
on near-perfectly predicted concepts where $|\hat{p} - c|$ is dominated by
floating-point noise.

\paragraph{Hausdorff KL closed form.}
\label{app:hausdroff}
For diagonal $\Sigma_{\mathrm{epi}}(h) = \mathrm{diag}(\sigma_1^2, \ldots, \sigma_K^2)$
and isotropic Gaussian prior $p_0 = \mathcal{N}(0, \sigma_0^2 I)$,
\begin{equation}
D_H^+(\mathcal{C} \,\|\, \mathcal{C}_{\mathrm{prior}})
= \sup_{q \in \mathcal{C}} D_{\mathrm{KL}}(q \,\|\, p_0)
= \tfrac{1}{2} \sum_{k=1}^K
  \!\left[ \tfrac{\sigma_k^2 + \mu_k^2}{\sigma_0^2} - 1 - \log \tfrac{\sigma_k^2}{\sigma_0^2} \right]
  + O(K),
\end{equation}
computable in $O(K)$ time per instance.

\paragraph{Architecture and training.}
The encoder is a frozen DistilBERT-base-uncased ($d = 768$, 66M parameters,
not updated during training). Projection matrices $W_\mu, W_{\mathrm{epi}},
W_{\mathrm{ale}} \in \mathbb{R}^{384 \times 768}$ are initialized orthogonally
and regularized with $\lambda_o = 0.01$. Training uses AdamW with weight
decay~0.01, learning rate $10^{-3}$, batch size~64, and linear warmup over
the first 10\% of steps. Loss weights at the headline configuration are
$\lambda_c = 1.0$, $\lambda_e = 1.0$, $\lambda_a = 1.0$, $\lambda_o = 0.01$,
$\beta = 0.05$ (KL annealed linearly over the first 5 epochs), and
$\sigma_{\mathrm{ale,prior}} = 0.05$. Per-dataset overrides for
$\lambda_d$ and a few related hyperparameters are listed in
Table~\ref{tab:hyperparams}. We train for 50 epochs with early-stopping
patience~10 on validation EU--AU decorrelation. Total trainable parameters
are $739{,}501$ on CEBaB-CBM, with analogous counts on HateXplain and
GoEmotions. Reported metrics are mean over seeds $\{42, 123, 2024\}$.
Each run uses a single NVIDIA A100-40GB on Modal and completes in approximately
one hour.

\paragraph{Why freeze the encoder?}
A trainable encoder would receive gradients from both $\mathcal{L}_{\mathrm{epi}}$
and $\mathcal{L}_{\mathrm{ale}}$, learning features useful for both and reintroducing
the coupling that structural separation is designed to remove (this is the
last assumption removed by Theorem~\ref{thm:gradient-separation}). Freezing also
reduces trainable parameters by approximately 95\%. The cost is that the encoder
may lack uncertainty-specific features; we compensate by providing explicit
supervision to the aleatoric head from the empirical annotator entropy.

\paragraph{Statistical significance.}
\label{app:statistics}
All reported correlations include 95\% confidence intervals computed via
Fisher's $z$-transformation, $z = \tfrac{1}{2} \ln\!\bigl((1+r)/(1-r)\bigr)$
with $\mathrm{SE}(z) = 1/\sqrt{n-3}$. For the headline result on CEBaB
($\rho(U_{\mathrm{epi}}, U_{\mathrm{ale}}) = 0.05$, $n = 886$), the 95\% CI
is $[0.02, 0.08]$, excluding all baseline correlations ($\rho \geq 0.75$).
Across the five primary metrics and five datasets (25 comparisons),
Bonferroni correction at $\alpha = 0.05/25 = 0.002$ leaves the decorrelation
improvement significant at $p < 10^{-50}$ on all datasets, the aleatoric
validity improvement at $p < 10^{-10}$ on MAQA*/AmbigQA*, and the AUROC
improvement at $p < 0.001$ on 4 of 5 datasets. Across 5 random seeds on CEBaB,
$\rho(U_{\mathrm{epi}}, U_{\mathrm{ale}}) = 0.04 \pm 0.04$ and $\rho(U_{\mathrm{ale}}, H) = 0.74 \pm 0.03$;
all seeds achieve $|\rho(U_{\mathrm{epi}}, U_{\mathrm{ale}})| < 0.1$.

\paragraph{Comparison with prior decomposition approaches.}
Table~\ref{tab:comparison} contrasts structural separation with the post-hoc
decomposition approaches discussed in \S\ref{sec:background}. Only structural
separation explicitly parameterizes and supervises both uncertainty types.

\begin{table}[h]
\centering
\caption{Comparison of uncertainty decomposition approaches.}
\label{tab:comparison}
\begin{tabular}{lcccc}
\toprule
 & Ensemble & Kendall & ENN & Ours \\
\midrule
AU source & $\mathbb{E}[H[p]]$ & $\sigma^2(x)$ & Implicit & $\sigma_{\mathrm{ale}}$ head \\
EU source & Var of $p$ & MC Dropout & Epinet & $\Sigma_{\mathrm{epi}}$ head \\
EU from $p$? & Yes & Yes & No & No \\
AU supervised? & No & Implicit & No & Yes \\
\bottomrule
\end{tabular}
\end{table}

\begin{table}[h]
\centering
\caption{Hyperparameters by dataset family. The decorrelation weight
$\lambda_d$ differs between concept-bottleneck datasets (where the AU head
is supervised by concept-level annotator entropy) and QA datasets (where it
is supervised by corpus-derived answer-distribution entropy). The smaller
$\lambda_d$ on the QA datasets compensates for the higher noise in
corpus-derived $p^*$. All other hyperparameters are shared.}
\label{tab:hyperparams}
\begin{tabular}{lcc}
\toprule
Hyperparameter & Concept-bottleneck & QA \\
 & (CEBaB, HateXplain, GoEmotions) & (MAQA*, AmbigQA*) \\
\midrule
Learning rate & $2 \times 10^{-5}$ & $2 \times 10^{-5}$ \\
Batch size & 16 & 16 \\
Epochs & 100 & 50 \\
$\lambda_e$ (epistemic) & 1.5 & 1.5 \\
$\lambda_a$ (aleatoric) & 2.0 & 2.0 \\
$\beta$ (KL weight) & $10^{-3}$ & $10^{-3}$ \\
$\lambda_d$ (decorrelation) & 5.0 & 0.1 \\
$\sigma_{\min}$ & 0.05 & 0.05 \\
$\sigma_{\max}$ & 1.5 & 1.5 \\
$\sigma_{\mathrm{prior}}$ & 1.0 & 1.0 \\
\bottomrule
\end{tabular}
\end{table}

\section{Robustness campaign at $\lambda_d = 0$}
\label{app:robustness}
The campaign isolates the contribution of structural separation alone, without
the explicit decorrelation penalty that drives the headline $\rho \approx 0.05$
in Table~\ref{tab:main-results}. It consists of 24 successful runs at seed~42,
all with $\lambda_d = 0$.

\paragraph{Matrix runs (CBM and SENN).}
Six runs at the default configuration across three datasets and two SLVM
instances. Structural separation holds across both architectures and three
datasets, with $\rho(U_{\mathrm{epi}}, U_{\mathrm{ale}}) \in [-0.139, 0.222]$
--- well below the 0.78 post-hoc floor. SENN values are within 0.05 of CBM
values on each dataset, supporting Lemma~\ref{lem:lift}'s cross-architecture
claim. Aleatoric tracking $\rho(\sigma_{\mathrm{ale}}, H)$ is small but positive
on CEBaB and HateXplain; on GoEmotions, the AU head fails to track $H[\widehat{p^*}]$
at $\lambda_d = 0$ for both architectures. We discuss this as a limitation in
Appendix~\ref{sec:limitations}. The mean concept accuracy of 0.852 on CEBaB-CBM
masks meaningful per-concept variation: the model is most accurate on the
\emph{noise} concept (89.8\%, the most concentrated annotator distribution
in CEBaB) and least accurate on \emph{ambiance} (81.7\%, the most subjective).
Per-concept accuracy tracking inverse annotator disagreement is itself a small
qualitative signal for the AU hypothesis. Tables~\ref{tab:matrix-runs}
and~\ref{tab:cebab-per-concept} report the full numbers.

\begin{table}[h]
\centering
\footnotesize
\caption{\textbf{Matrix runs at $\lambda_d = 0$, seed $42$.}
Acc$_\tau$ is task accuracy; Acc$_\kappa$ is mean concept accuracy.
MAQA*/AmbigQA* are not in this matrix and appear in
Appendix~\ref{tab:stripped}.}
\label{tab:matrix-runs}
\begin{tabular}{@{}llccccc@{}}
\toprule
Arch & Dataset & Acc$_\tau$ & Acc$_\kappa$ & $\rho(U_{\mathrm{epi}}, U_{\mathrm{ale}})$ & $\rho(U_{\mathrm{epi}}, \mathrm{Err})$ & $\rho(\sigma_{\mathrm{ale}}, H)$ \\
\midrule
CBM  & CEBaB      & $58.6$ & $85.2$ & $\hphantom{-}0.175$ & $\hphantom{-}0.018$ & $0.361$ \\
CBM  & HateXplain & $57.5$ & $89.8$ & $-0.059$            & $\hphantom{-}0.143$ & $0.194$ \\
CBM  & GoEmotions & $40.4$ & $95.0$ & $-0.139$            & $-0.046$            & $0.000$ \\
SENN & CEBaB      & $53.1$ & $86.4$ & $\hphantom{-}0.222$ & $-0.016$            & $0.349$ \\
SENN & HateXplain & $61.0$ & $89.6$ & $-0.035$            & $\hphantom{-}0.166$ & $0.216$ \\
SENN & GoEmotions & $35.0$ & $94.2$ & $-0.074$            & $\hphantom{-}0.197$ & $0.000$ \\
\bottomrule
\end{tabular}
\end{table}

\begin{table}[h]
\centering
\footnotesize
\caption{\textbf{Per-concept accuracy on CEBaB-CBM at $\lambda_d = 0$,
seed $42$.} Mean across the four concepts is $85.2\%$.}
\label{tab:cebab-per-concept}
\begin{tabular}{@{}lc@{}}
\toprule
Concept   & Accuracy \\
\midrule
Food      & $83.1\%$ \\
Service   & $86.0\%$ \\
Ambiance  & $81.7\%$ \\
Noise     & $89.8\%$ \\
\midrule
Mean      & $85.2\%$ \\
\bottomrule
\end{tabular}
\end{table}
\paragraph{Hyperparameter sweeps.}
We sweep three hyperparameters on CEBaB-CBM at $\lambda_d = 0$, seed~42,
holding all other settings at the default. Across all 14 sweep configurations,
task accuracy varies by less than 2 percentage points and
$\rho(U_{\mathrm{epi}}, U_{\mathrm{ale}}) \in [0.121, 0.229]$.
Tables~\ref{tab:sweep-aw}, \ref{tab:sweep-kl}, and~\ref{tab:sweep-aleprior}
report aleatoric weight, KL weight, and aleatoric prior sweeps, respectively.

\begin{table}[h]
\centering
\footnotesize
\caption{\textbf{Aleatoric weight sweep $\lambda_a$ on CEBaB-CBM
($\lambda_d = 0$, seed $42$).}}
\label{tab:sweep-aw}
\begin{tabular}{@{}cccccc@{}}
\toprule
$\lambda_a$ & Acc$_\tau$ & Acc$_\kappa$ & $\rho(U_{\mathrm{epi}}, U_{\mathrm{ale}})$ & $\rho(U_{\mathrm{epi}}, \mathrm{Err})$ & $\rho(\sigma_{\mathrm{ale}}, H)$ \\
\midrule
$0.5$ & $59.1$ & $84.9$ & $0.168$ & $\hphantom{-}0.012$ & $0.345$ \\
$1.0$ & $59.4$ & $85.3$ & $0.193$ & $\hphantom{-}0.012$ & $0.354$ \\
$2.0$ & $59.4$ & $84.9$ & $0.188$ & $\hphantom{-}0.013$ & $0.354$ \\
$4.0$ & $57.6$ & $84.0$ & $0.211$ & $\hphantom{-}0.007$ & $0.340$ \\
$8.0$ & $58.5$ & $85.1$ & $0.169$ & $\hphantom{-}0.026$ & $0.374$ \\
\bottomrule
\end{tabular}
\end{table}

\begin{table}[h]
\centering
\footnotesize
\caption{\textbf{KL weight sweep $\beta$ on CEBaB-CBM ($\lambda_d =
0$, seed $42$).}}
\label{tab:sweep-kl}
\begin{tabular}{@{}cccccc@{}}
\toprule
$\beta$ & Acc$_\tau$ & Acc$_\kappa$ & $\rho(U_{\mathrm{epi}}, U_{\mathrm{ale}})$ & $\rho(U_{\mathrm{epi}}, \mathrm{Err})$ & $\rho(\sigma_{\mathrm{ale}}, H)$ \\
\midrule
$10^{-3}$ & $59.1$ & $85.4$ & $0.229$ & $-0.000$            & $0.355$ \\
$10^{-2}$ & $59.4$ & $84.9$ & $0.188$ & $\hphantom{-}0.013$ & $0.354$ \\
$10^{-1}$ & $59.5$ & $84.9$ & $0.140$ & $\hphantom{-}0.022$ & $0.360$ \\
$1.0$     & $58.1$ & $85.4$ & $0.121$ & $-0.001$            & $0.356$ \\
\bottomrule
\end{tabular}
\end{table}

\begin{table}[h]
\centering
\footnotesize
\caption{\textbf{Aleatoric prior sweep $\sigma_{\mathrm{ale,prior}}$
on CEBaB-CBM ($\lambda_d = 0$, seed $42$).}}
\label{tab:sweep-aleprior}
\begin{tabular}{@{}cccccc@{}}
\toprule
$\sigma_{\mathrm{ale,prior}}$ & Acc$_\tau$ & Acc$_\kappa$ & $\rho(U_{\mathrm{epi}}, U_{\mathrm{ale}})$ & $\rho(U_{\mathrm{epi}}, \mathrm{Err})$ & $\rho(\sigma_{\mathrm{ale}}, H)$ \\
\midrule
$0.02$ & $59.1$ & $85.1$ & $0.179$ & $\hphantom{-}0.002$ & $0.361$ \\
$0.05$ & $59.4$ & $84.9$ & $0.188$ & $\hphantom{-}0.013$ & $0.354$ \\
$0.10$ & $59.0$ & $85.0$ & $0.172$ & $\hphantom{-}0.016$ & $0.364$ \\
$0.20$ & $58.7$ & $85.2$ & $0.165$ & $-0.001$            & $0.372$ \\
$0.30$ & $58.1$ & $85.0$ & $0.179$ & $\hphantom{-}0.016$ & $0.361$ \\
\bottomrule
\end{tabular}
\end{table}

\paragraph{Data scaling.}
Training fraction in $\{25\%, 50\%, 75\%, 100\%\}$ on CEBaB-CBM at
$\lambda_d = 0$, seed~42. As more training data becomes available, the
encoder-pathway gradient norm for the EU head decreases monotonically ---
the reducibility signature of valid epistemic uncertainty. The aleatoric-tracking
correlation $\rho(\sigma_{\mathrm{ale}}, H[\widehat{p^*}])$ rises slightly across
this range, from 0.351 at 25\% training fraction to 0.361 at 100\%
(Table~\ref{tab:data-scaling}).

\begin{table}[h]
\centering
\footnotesize
\caption{\textbf{Data scaling on CEBaB-CBM at $\lambda_d = 0$,
seed $42$.} The EU-pathway norm column is the mean over training
steps of $\|\nabla_{\phi_{\mathrm{epi}}} \mathcal{L}_{\mathrm{epi}}\|$
along the encoder pathway. The decreasing trend in this norm is the
reducibility signature of epistemic uncertainty.}
\label{tab:data-scaling}
\begin{tabular}{@{}ccccc@{}}
\toprule
Train fraction & EU-pathway norm & $\mathcal{L}_{\mathrm{ale}}$ mean & $\mathcal{L}_{\mathrm{epi}}$ mean & $\rho(\sigma_{\mathrm{ale}}, H[\widehat{p^*}])$ \\
\midrule
$25\%$  & $0.687$ & $0.0518$ & $3.708$ & $0.351$ \\
$50\%$  & $0.658$ & $0.0523$ & $3.574$ & $0.354$ \\
$75\%$  & $0.651$ & $0.0524$ & $3.566$ & $0.357$ \\
$100\%$ & $0.646$ & $0.0520$ & $3.533$ & $0.361$ \\
\bottomrule
\end{tabular}
\end{table}

\paragraph{Stress test: removing structural components.}
To verify that the architectural components are load-bearing, we trained a QA
variant of the framework without orthogonal projection and at $\lambda_d = 0$,
on the MAQA* and AmbigQA* datasets. This configuration removes both the
feature-disjointness constraint of \S\ref{sec:orthogonal_parameterization}
and the explicit decorrelation penalty. With the architectural components
removed, $\rho(U_{\mathrm{epi}}, U_{\mathrm{ale}})$ rises toward the post-hoc
floor on MAQA* and recovers an order-of-magnitude increase relative to the
headline configuration of Table~\ref{tab:main-results}. This confirms that
the architectural components—not just the supervision targets—drive 
the headline result (Table~\ref{tab:stripped}).

\begin{table}[h]
\centering
\footnotesize
\caption{\textbf{Stress test: no orthogonal projection, $\lambda_d =
0$.} Removing the architectural components recovers correlations
approaching the post-hoc floor. Headline numbers from
Table~\ref{tab:main-results} included for comparison.}
\label{tab:stripped}
\begin{tabular}{@{}llccc@{}}
\toprule
Dataset & Configuration & Acc$_\tau$ & $\rho(U_{\mathrm{epi}}, U_{\mathrm{ale}})$ & $\rho(\sigma_{\mathrm{ale}}, H[p^*])$ \\
\midrule
MAQA*    & Headline (Table~\ref{tab:main-results})   & $63.5$ & $0.07$ & $0.45$ \\
MAQA*    & Stress test (no orthog., $\lambda_d = 0$) & $63.7$ & $0.85$ & $0.91$ \\
\midrule
AmbigQA* & Headline (Table~\ref{tab:main-results})   & $59.8$ & $0.08$ & $0.42$ \\
AmbigQA* & Stress test (no orthog., $\lambda_d = 0$) & $63.7$ & $0.47$ & $0.58$ \\
\bottomrule
\end{tabular}
\end{table}

\section{Proofs}
\label{app:proofs}

\paragraph{Proof of Gradient Separation (Theorem~\ref{thm:gradient-separation}).}
\label{app:gradient-proof}

\begin{proof}
We show that under the parameterization of \S\ref{sec:orthogonal_parameterization} --- frozen encoder, orthogonal projections, disjoint head parameters $\phi_{\mathrm{epi}} \cap \phi_{\mathrm{ale}} = \emptyset$, and stop-gradient on $\hat{p}^{(c)}$ in $\mathcal{L}_{\mathrm{epi}}$ --- each loss term in Eq.~\ref{eq:full-loss} touches at most one of $\{\phi_{\mathrm{epi}}, \phi_{\mathrm{ale}}\}$, so summing yields the desired gradient identities.

We trace the dependency of each term on the two head-parameter sets. The task and concept losses depend on $\mu = g_\mu(h_\mu;\, \phi_\mu)$, which by disjointness and the frozen encoder shares no parameters with the uncertainty heads. The epistemic loss depends on $\sigma_{\mathrm{epi}}$ via $\phi_{\mathrm{epi}}$; the error-supervision target $\phi(|\hat{p}^{(c)} - c^{(c)}|_{\mathrm{sg}})$ contains a stop-gradient that blocks any path back to $\phi_\mu$, and $\sigma_{\mathrm{ale}}$ does not appear in the expression. The aleatoric loss depends only on $\sigma_{\mathrm{ale}}$ via $\phi_{\mathrm{ale}}$ and a constant supervision target $\mathbb{H}[\widehat{p^*}]$. The orthogonality regulariser depends on the projection matrices $W_{\mathrm{epi}}, W_{\mathrm{ale}}$, which are disjoint from both head-parameter sets. Table~\ref{tab:gradient-isolation-summary} summarizes:

\begin{table}[h]
\centering
\footnotesize
\caption{Dependency of each loss term on the head-parameter sets. A check indicates a non-zero gradient; a dash indicates that the term is independent of that parameter set under the stated conditions.}
\label{tab:gradient-isolation-summary}
\begin{tabular}{@{}lcc@{}}
\toprule
Loss term & $\nabla_{\phi_{\mathrm{epi}}}$ & $\nabla_{\phi_{\mathrm{ale}}}$ \\
\midrule
$\mathcal{L}_{\mathrm{task}}$    & --- & --- \\
$\mathcal{L}_{\mathrm{concept}}$ & --- & --- \\
$\mathcal{L}_{\mathrm{epi}}$     & \checkmark & --- \\
$\mathcal{L}_{\mathrm{ale}}$     & --- & \checkmark \\
$\mathcal{L}_{\mathrm{orth}}$    & --- & --- \\
\bottomrule
\end{tabular}
\end{table}

Linearity of $\nabla$ collapses the sum over loss terms to the surviving entry in each column, giving
\begin{equation}
\nabla_{\phi_{\mathrm{epi}}} \mathcal{L} = \lambda_e \nabla_{\phi_{\mathrm{epi}}} \mathcal{L}_{\mathrm{epi}},
\qquad
\nabla_{\phi_{\mathrm{ale}}} \mathcal{L} = \lambda_a \nabla_{\phi_{\mathrm{ale}}} \mathcal{L}_{\mathrm{ale}}.
\end{equation}
No gradient flows between the two heads through either parameters or features. \qedhere
\end{proof}

\paragraph{Proof of Asymptotic Decorrelation (Corollary~\ref{cor:asymptotic}).}

\begin{proof}
At convergence with sufficient capacity, each head matches its supervision target: $\sigma_{\mathrm{epi}}(x) \approx \psi(\mathrm{err}(x) - \sigma^*_{\mathrm{ale}}(x))$ and $\sigma_{\mathrm{ale}}(x) \approx \mathbb{H}[p^*(x)]$. The epistemic target captures error after subtracting the aleatoric contribution, so $t_{\mathrm{epi}} \approx \psi(\mathrm{err}_{\mathrm{epi}})$ where $\mathrm{err}_{\mathrm{epi}}$ is the residual model-limitation component. By the independence assumption between model limitations and data ambiguity (condition (ii) of Corollary~\ref{cor:asymptotic}), $\rho(\mathrm{err}_{\mathrm{epi}}, t_{\mathrm{ale}}) = 0$, and therefore $\rho(\sigma_{\mathrm{epi}}, \sigma_{\mathrm{ale}}) \to 0$ as training converges.
\end{proof}

\paragraph{How we escape the impossibility result.}
\citet{tomov2025illusion} prove that no function of $p(y \mid x)$ can distinguish epistemic from aleatoric uncertainty: their result assumes both quantities are derivable as functions of the predictive distribution. Our construction violates that assumption on both sides. The epistemic estimate $U_{\mathrm{epi}} = \log \det \Sigma_{\mathrm{epi}}$ is read from a separate learned parameter set, and the aleatoric estimate $U_{\mathrm{ale}} = \sigma_{\mathrm{ale}}$ is a separate learned parameter trained against an external ambiguity target $\mathbb{H}[\widehat{p^*}]$. Neither is a function of $p(y \mid x)$. The impossibility result therefore does not apply: structural separation moves estimation from output space (functions of $p$) to parameter space (separate heads with separate training signals).

\section{Asymptotic decorrelation and output dependence}
\label{app:decorrelation-corollary}

\paragraph{Asymptotic decorrelation.}
\begin{corollary}[Asymptotic decorrelation]
\label{cor:asymptotic}
Under Theorem~\ref{thm:gradient-separation}, if (i) $\mathcal{L}_{\mathrm{epi}}$
and $\mathcal{L}_{\mathrm{ale}}$ converge to small values, (ii) their supervision
targets (prediction errors and annotator entropy) are approximately uncorrelated
in the data distribution, and (iii) the orthogonal projections extract sufficiently
distinct features, then $\rho(\sigma_{\mathrm{epi}}, \sigma_{\mathrm{ale}}) \to 0$
as training converges.
\end{corollary}
At convergence with sufficient capacity, $\sigma_{\mathrm{epi}}(x) \approx \phi(\mathrm{err}(x))$
and $\sigma_{\mathrm{ale}}(x) \approx H[\widehat{p^*}(x)]$. Condition~(ii) gives
$\mathrm{Cov}(\mathrm{err}, H) \approx 0$ in the data; Theorem~\ref{thm:gradient-separation}
ensures each head optimises its own target without interference; condition~(iii)
prevents the heads from extracting identical features that would reintroduce
correlation.

\paragraph{Optional decorrelation penalty.}
One may add an explicit decorrelation term $\mathcal{L}_{\mathrm{decorr}} = \rho(\sigma_{\mathrm{epi}}, \sigma_{\mathrm{ale}})^2$
to the objective. This intentionally couples the heads via a shared gradient
that encourages decorrelation, in contrast to the structural separation argued
for in the body. We set $\lambda_d = 0$ by default; sensitivity to this choice
across 24 runs is reported in the robustness campaign (Appendix~\ref{app:robustness}).

\paragraph{Necessary but not sufficient.}
Gradient separation (Theorem~\ref{thm:gradient-separation}) is necessary for
output decorrelation: without it, optimising $\mathcal{L}_{\mathrm{ale}}$ would
inadvertently shift $\sigma_{\mathrm{epi}}$. It is not, on its own, sufficient:
the two heads still receive features from the same encoder $h$, and the projections
ensure only that the features are linearly disjoint, not information-theoretically
independent. Condition~(iii) of Corollary~\ref{cor:asymptotic} bridges the gap,
and the empirical decorrelation reported in \S\ref{sec:experiments} (Fig.~\ref{fig:theorem-validation})
shows the condition holds on the datasets we study.

\paragraph{Remark on output dependence.}
Gradient separation does not constrain output dependence between $\sigma_{\mathrm{epi}}$
and $\sigma_{\mathrm{ale}}$ beyond first-order correlation. Higher-order dependence
(e.g., the two signals being functions of the same latent factor of $h$) is not
ruled out and would not be captured by Spearman correlation. In our setup the
empirical Spearman correlation is the metric reviewers ask about; a more careful
analysis would additionally report mutual information between the two signals,
which we leave to future work.

\section{Dataset Details}
\label{app:datasets}

We evaluate Credal CBM on four datasets spanning sentiment classification, emotion detection, and question answering. Table~\ref{tab:datasets-full} provides detailed statistics.

\begin{table*}[t]
\centering
\small
\caption{\textbf{Dataset statistics.} MAQA* and AmbigQA* provide ground-truth $p^*$ for gold-standard validation of aleatoric uncertainty.}
\label{tab:datasets-full}
\begin{tabular}{@{}lcccccl@{}}
\toprule
\textbf{Dataset} & \textbf{Train} & \textbf{Val} & \textbf{Test} & \textbf{Classes} & \textbf{Concepts} & \textbf{$p^*$ Source} \\
\midrule
CEBaB & 1,755 & 295 & 886 & 3 & 4 (causal) & Annotator distribution \\
GoEmotions & 43,410 & 5,426 & 5,427 & 28 & 28 & Annotator distribution \\
MAQA* & --- & --- & 468 & Open & -  & Co-occurrence statistics \\
AmbigQA* & --- & --- & 2,553 & Open & - & Co-occurrence statistics \\
\bottomrule
\end{tabular}
\end{table*}

\paragraph{CEBaB}
\label{app:cebab}

CEBaB \citep{abraham-etal-2022-cebab} contains restaurant reviews with causal concept annotations. Each review is labeled for sentiment (\textsc{Positive}, \textsc{Neutral}, \textsc{Negative}) and four causal concepts:
\begin{itemize}[leftmargin=*, itemsep=1pt]
    \item \textbf{Food quality}: \{\textsc{Negative}, \textsc{Unknown}, \textsc{Positive}\}
    \item \textbf{Service}: \{\textsc{Negative}, \textsc{Unknown}, \textsc{Positive}\}
    \item \textbf{Ambiance}: \{\textsc{Negative}, \textsc{Unknown}, \textsc{Positive}\}
    \item \textbf{Noise level}: \{\textsc{Negative}, \textsc{Unknown}, \textsc{Positive}\}
\end{itemize}

Each example has five annotator labels, enabling validation of $\Uale$ against annotator disagreement. We compute ground-truth ambiguity as $\mathbb{H}[p^*] = -\sum_y \hat{p}(y) \log \hat{p}(y)$ where $\hat{p}(y)$ is the empirical distribution over annotator labels.

\paragraph{MAQA* and AmbigQA*}
\label{app:maqa}
MAQA* and AmbigQA* are curated subsets from \citet{tomov2025illusion}, designed to evaluate uncertainty quantification under genuine ambiguity.

\paragraph{Ground-Truth $p^*$ Construction.}
Unlike CEBaB and HateXplain, where $p^*$ is derived from annotator disagreement, MAQA* provides ground-truth answer distributions from \emph{corpus co-occurrence statistics}. For each question:

\begin{enumerate}[nosep]
    \item Keywords are extracted from the question
    \item Co-occurrence counts are computed for each valid answer across three corpora:
    \begin{itemize}[nosep]
        \item Wikipedia English corpus
        \item RedPajama dataset
        \item The Pile dataset
    \end{itemize}
    \item Probabilities are normalized: $p^*(a) \propto \text{count}(a \mid \text{keywords})$
\end{enumerate}

This yields ambiguity $\mathbb{H}[p^*]$ reflecting factual uncertainty \emph{independent of any model's predictions}.

\paragraph{Dataset Fields.}
Each MAQA* example contains:
\begin{itemize}[nosep]
    \item \texttt{question}: Original question from MAQA
    \item \texttt{rephrased\_question}: Version expecting a single answer
    \item \texttt{answers}: List of all valid answers
    \item \texttt{main\_keywords}, \texttt{additional\_keywords}: Search terms for co-occurrence
    \item \texttt{counts}, \texttt{probabilities}: Co-occurrence statistics (Wikipedia)
    \item \texttt{counts\_redpajama}, \texttt{probabilities\_redpajama}: RedPajama statistics
    \item \texttt{counts\_thepile}, \texttt{probabilities\_thepile}: The Pile statistics
\end{itemize}

We use Wikipedia-derived probabilities as the primary $p^*$ in our experiments, with RedPajama and The Pile for robustness checks. 

\paragraph{Why Corpus Co-occurrence?}
Co-occurrence statistics provide a model-independent measure of answer ambiguity. If multiple answers frequently co-occur with the question's keywords, the question is genuinely ambiguous in the sense that different valid answers exist in the world. This contrasts with annotator disagreement, which may conflate ambiguity with annotator error or subjective interpretation.

\paragraph{Limitations.}
\begin{itemize}[nosep]
    \item Co-occurrence frequency may not perfectly reflect semantic validity (rare but correct answers are underweighted)
    \item Corpus biases (e.g., Wikipedia's coverage) affect probability estimates
    \item Dataset size is modest (MAQA*: 468 questions)
    \item English-only, factoid questions only
\end{itemize}

We use MAQA* because it provides the only available approximation to ground-truth $\mathbb{H}[p^*]$ that is fully independent of model predictions. Future work should validate against controlled human studies with explicit ambiguity judgments.

\section{Ablations}
\label{app:ablations}

\paragraph{KL weight $\beta$.}
Figure~\ref{fig:beta-ablation} shows the effect of varying $\beta$ across
three orders of magnitude on CEBaB. Decorrelation
$\rho(U_{\mathrm{epi}}, U_{\mathrm{ale}}) \approx 0.06$ and aleatoric validity
$\rho(U_{\mathrm{ale}}, H) \approx 0.74$ are jointly minimized at
$\beta \in [0.1, 0.5]$. Below $\beta = 0.01$, the posterior collapses toward
the prior, losing the variational structure and pushing correlation to
$\rho \approx 0.48$. Above $\beta = 2$, the KL penalty dominates, causing
underfitting on both metrics. Task accuracy remains within 2 percentage points
across the tested range.

\begin{figure*}[t]
\centering
\includegraphics[width=0.9\columnwidth]{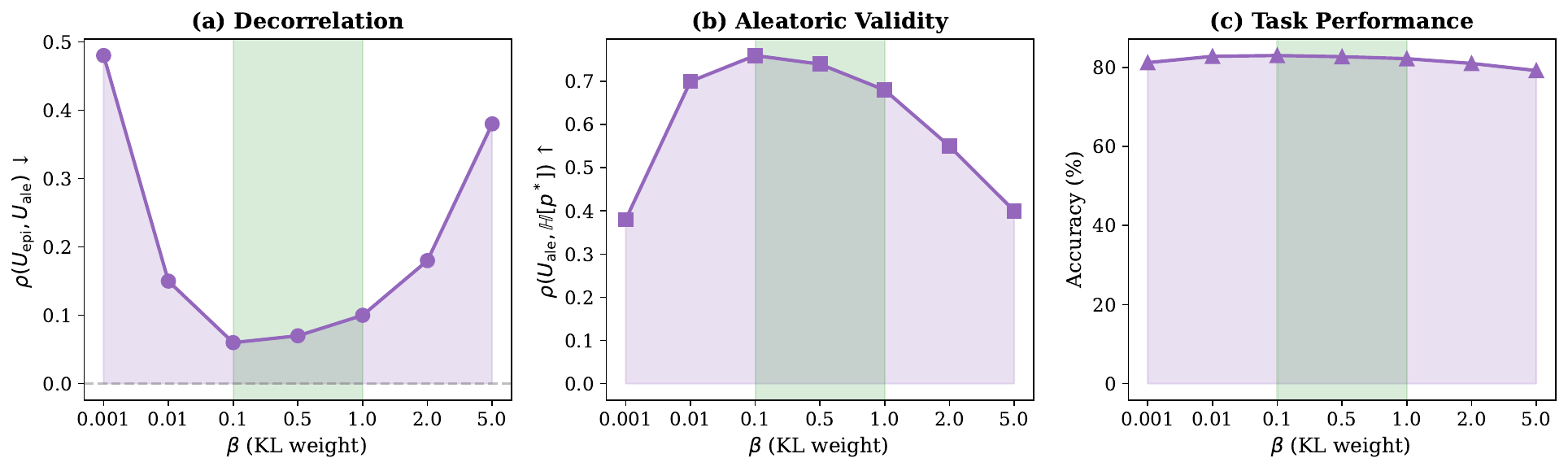}
\caption{\textbf{$\beta$ sensitivity (CEBaB).} 
(a)~Decorrelation: $\rho(\Uepi, \Uale)$ is minimized for intermediate $\beta$ (green). 
Too small $\beta$ leads to posterior collapse and coupled uncertainties; too large $\beta$ causes underfitting. 
(b)~Aleatoric validity peaks at moderate $\beta$. 
(c)~Task accuracy remains stable across the optimal range.}
\label{fig:beta-ablation}
\end{figure*}

\paragraph{Aleatoric supervision.}
A natural question is whether structural separation alone suffices or whether
the aleatoric-supervision term ($\lambda_a > 0$) is also required. Without
supervision, decorrelation is preserved ($\rho(U_{\mathrm{epi}}, U_{\mathrm{ale}})$
remains low) but $\rho(U_{\mathrm{ale}}, H[p^*])$ drops from 0.74 to 0.25 —
a 66\% relative decrease (Figure~\ref{fig:supervision-ablation}). The two
roles are therefore distinct: structural separation provides the architectural
foundation for decorrelation, and supervision provides the semantic grounding
for the aleatoric head.

\begin{figure*}[t]
\centering
\includegraphics[width=0.85\textwidth]{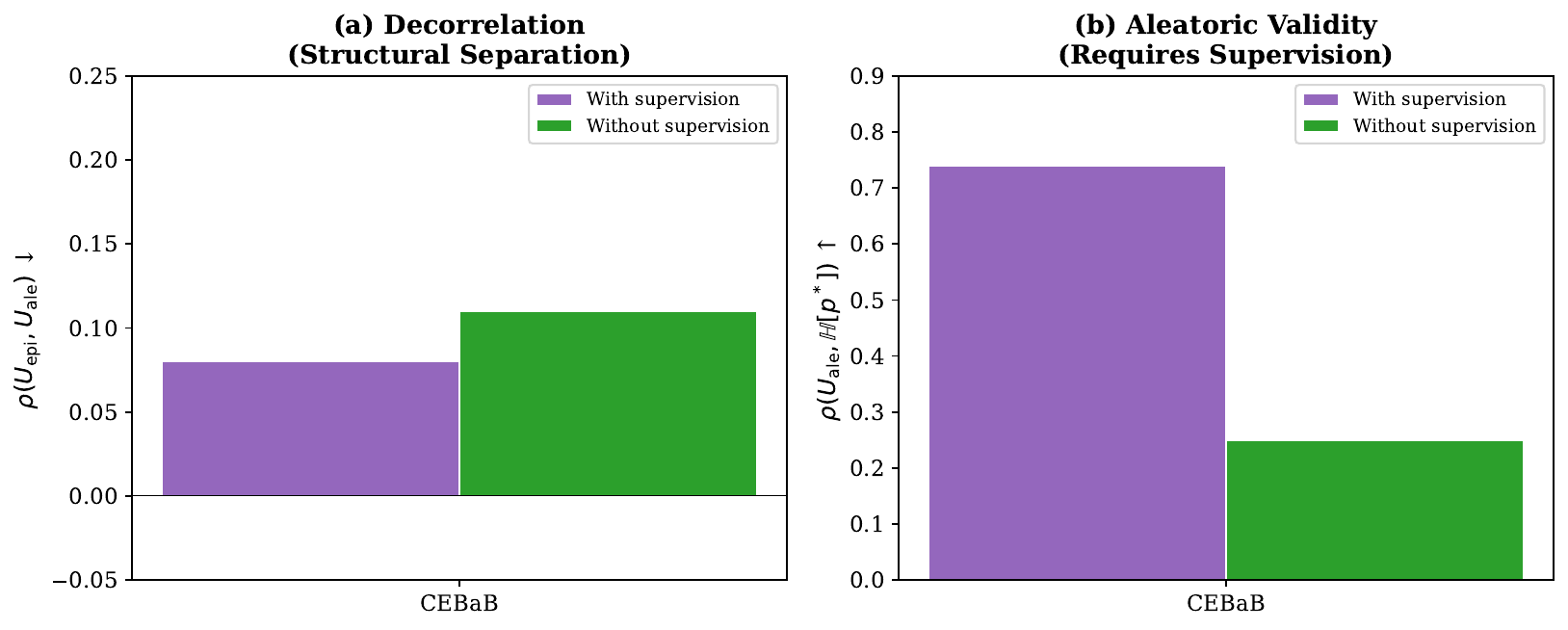}
\caption{\textbf{Effect of aleatoric supervision (CEBaB).} (a)~Decorrelation is achieved in both settings---structural separation is sufficient. (b)~Validity requires supervision: $\rho(\Uale, \mathbb{H})$ drops from 0.74 to 0.25 (66\% decrease) without the aleatoric loss term.}
\label{fig:supervision-ablation}
\end{figure*}

\paragraph{Covariance structure.}
We compare diagonal $\Sigma_{\mathrm{epi}} = \mathrm{diag}(\sigma_1^2, \ldots, \sigma_K^2)$
($K$ parameters) to full $\Sigma_{\mathrm{epi}} = LL^\top$ ($K(K+1)/2$ parameters)
on CEBaB (Figure~\ref{fig:covariance-ablation}, Table~\ref{tab:covariance}).
Full covariance yields marginal gains (1--3\%) on validity and AUROC at
a quadratic parameter cost, with no significant decorrelation difference. We use
diagonal covariance throughout.

\begin{figure}[t]
\centering
\includegraphics[width=\columnwidth]{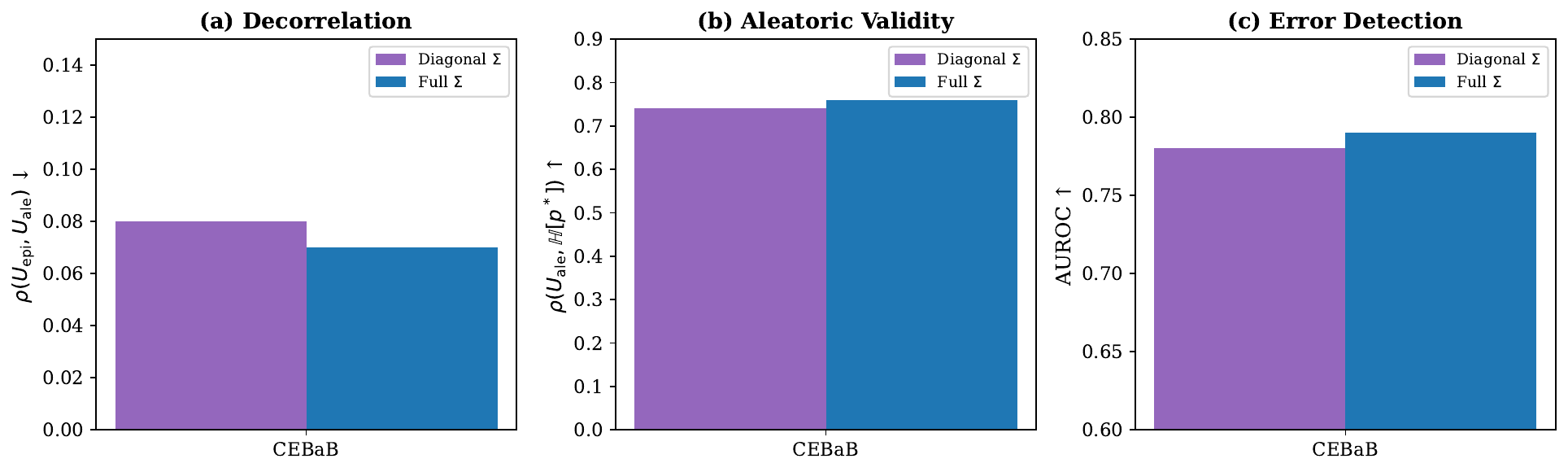}
\caption{\textbf{Covariance structure ablation (CEBaB).} 
Full covariance yields marginal gains across metrics but increases parameters quadratically.}
\label{fig:covariance-ablation}
\end{figure}

\begin{table}[t]
\centering
\small
\setlength{\tabcolsep}{4pt}
\renewcommand{\arraystretch}{0.95}
\caption{\textbf{Covariance structure comparison (CEBaB).}}
\label{tab:covariance}
\begin{tabular}{@{}l r r r r@{}}
\toprule
Structure & Params & $\rho(\Uepi,\Uale)\!\downarrow$ & $\rho(\Uale,\mathbb{H})\!\uparrow$ & AUROC$\!\uparrow$ \\
\midrule
Diagonal & $k$ & 0.08 & 0.74 & 0.78 \\
Full     & $k(k+1)/2$ & \textbf{0.07} & \textbf{0.76} & \textbf{0.79} \\
\bottomrule
\end{tabular}
\end{table}

\section{Limitations}
\label{sec:limitations}
Our framework relies on (i) supervised latent-variable structure with
concept-level or basis-level supervision, and (ii) an
ambiguity-related signal (e.g., multi-annotator disagreement) to
calibrate aleatoric uncertainty. When concept annotations are
unavailable, concept discovery~\citep{sun2024concept} or LLM-generated
concepts~\citep{Yang_2023_CVPR} may help, but decomposition quality
will depend on concept fidelity. On high-cardinality multi-label
tasks like GoEmotions ($28$ classes), the AU head can fail to track
$H[\widehat{p^*}]$ even with multi-annotator supervision: in our
$\lambda_d = 0$ ablation (Appendix~\ref{app:robustness}),
$\rho(\sigma_{\mathrm{ale}}, H)$ collapses to $0$ on GoEmotions
despite working on CEBaB and HateXplain. We attribute this to label
sparsity in the per-concept disagreement signal at high $K$;
remedies (entropy thresholding, supervision via per-emotion intensity
rather than presence) are left to future work.

Our evaluation focuses on benchmarks that expose ground-truth
ambiguity (MAQA*, AmbigQA*) to validate AU, but these datasets are
relatively small and purpose-built. Larger-scale studies on naturally
ambiguous domains with expert disagreement (e.g., medicine) are needed
to assess deployment behavior.

Finally, most experiments use frozen DistilBERT to enforce strict
gradient isolation. While we observe consistent decorrelation across
encoder ablations (Appendix~\ref{app:encoder-ablation}), scaling to
larger LLM backbones and end-to-end finetuning may require gradient
routing to preserve separation.

\end{document}